%% file: main.tex
\documentclass[10pt,twocolumn,letterpaper]{article}

\usepackage[pagenumbers]{cvpr} 

\input{preamble}

\definecolor{cvprblue}{rgb}{0.21,0.49,0.74}
\usepackage[pagebackref,breaklinks,colorlinks,allcolors=cvprblue]{hyperref}

\def\paperID{*****} 
\def\confName{CVPR}
\def\confYear{2026}

\title{\paper: Local Density-Aware Similarity Scoring for Biometrics}

\author{
Yiyang Su\textsuperscript{1}\quad
Minchul Kim\textsuperscript{1}\quad
Jie Zhu\textsuperscript{1}\quad
Christopher Perry\textsuperscript{1}\\
Feng Liu\textsuperscript{2}\quad
Anil Jain\textsuperscript{1}\quad
Xiaoming Liu\textsuperscript{1,3}\\
\textsuperscript{1} Michigan State University\quad
\textsuperscript{2} Drexel University\quad
\textsuperscript{3} University of North Carolina, Chapel Hill\\
{\tt\small \{suyiyan1, kimminc2, zhujie4, perrych2\}@msu.edu, fl397@drexel.edu, liuxm@cs.unc.edu}
}

\begin{document}
\maketitle
\input{sec/0_abstract}
\input{sec/1_intro}
\input{sec/2_related_work}
\input{sec/3_method}
\input{sec/4_experiments}
\input{sec/5_analyses}
\input{sec/6_conclusion}
{
    \small
    \bibliographystyle{ieeenat_fullname}
    \bibliography{sec/refs}
}

\input{sec/X_suppl}

\end{document}

%% file: preamble.tex
\usepackage{microtype}

\renewcommand{\paragraph}[1]{\vspace{.5em}\noindent\textbf{#1}}

\usepackage{xspace}
\newcommand{\paper}{LocalScore\xspace}

\usepackage{nicematrix}
\usepackage{pifont} 
\usepackage{siunitx} 
\usepackage{algorithm} 
\usepackage{algorithmic} 

\usepackage{amsthm} 
\theoremstyle{plain}
\newtheorem{theorem}{Theorem}

\newcommand{\smallpm}{{\scriptstyle \pm}}

\usepackage{listings}

\definecolor{codegreen}{rgb}{0,0.6,0}
\definecolor{codegray}{rgb}{0.5,0.5,0.5}
\definecolor{codepurple}{rgb}{0.58,0,0.82}
\definecolor{backcolour}{rgb}{0.95,0.95,0.92}

\lstdefinestyle{modern}{
    basicstyle=\ttfamily\footnotesize,    
    commentstyle=\color{codegreen},
    keywordstyle=\color{magenta},
    numberstyle=\tiny\color{codegray},
    stringstyle=\color{codepurple},
    backgroundcolor=\color{black!5},     
    frame=single,                         
    rulecolor=\color{black},              
    breaklines=true,                      
    columns=flexible,                     
    showstringspaces=false,               
    upquote=true,                         
    morekeywords={int, float, if, else, return}  
}

\newcounter{algo}

%% file: sec/0_abstract.tex
\begin{abstract}

Open-set biometrics faces challenges with probe subjects who may not be enrolled in the gallery, as traditional biometric systems struggle to detect these non-mated probes. Despite the growing prevalence of multi-sample galleries in real-world deployments, most existing methods collapse intra-subject variability into a single global representation, leading to suboptimal decision boundaries and poor open-set robustness. To address this issue, we propose \paper, a simple yet effective scoring algorithm that explicitly incorporates the local density of the gallery feature distribution using the \(k\)-th nearest neighbors. \paper is architecture-agnostic, loss-independent, and incurs negligible computational overhead, making it a plug-and-play solution for existing biometric systems.  Extensive experiments across multiple modalities demonstrate that \paper consistently achieves substantial gains in open-set retrieval (FNIR@FPIR reduced from \(53\%\) to \(40\%\)) and verification (TAR@FAR improved from \(51\%\) to \(74\%\)). We further provide theoretical analysis and empirical validation explaining when and why the method achieves the most significant gains based on dataset characteristics.
\end{abstract}

%% file: sec/1_intro.tex
\section{Introduction}
\label{sec:intro}
Biometrics involves identifying individuals through physical or behavioral characteristics, such as facial features~\cite{Jain2025}, gait patterns~\cite{shen2022comprehensive}, and body shape~\cite{zahra2023person}. It is integral to applications ranging from access control to surveillance~\cite{ross2019some}. At the core, biometric systems compute similarity scores to assess how closely a probe matches enrolled subjects in a gallery. Verification determines whether a probe matches a specific gallery subject and retrieval ranks all gallery subjects by similarity to the probe.

In biometric retrieval, early work focused on closed-set retrieval, where all probe subjects are assumed to exist in the gallery (\ie, mated). In contrast, open-set search represents a more realistic generalization, as the model must accurately {\it identify} mated probes while {\it rejecting} non-mated ones. These requirements make open-set retrieval inherently more challenging than closed-set retrieval, which only concerns within-gallery identification. Open-set evaluation has been recognized as one of the most challenging scenarios in biometrics~\cite{gunther2017unconstrained}. While state-of-the-art (SoTA) whole-body biometric system achieving over \(90\%\) closed-set Rank-20 accuracy, when only 1\% of non-mated probes are incorrectly identified as mated, it still fails to correctly identify about 55\% of mated probes~\cite{liu2024farsight}. Hence, in this paper, we focus on open-set retrieval and verification.

\begin{figure}
    \centering
    \includegraphics[width=\linewidth]{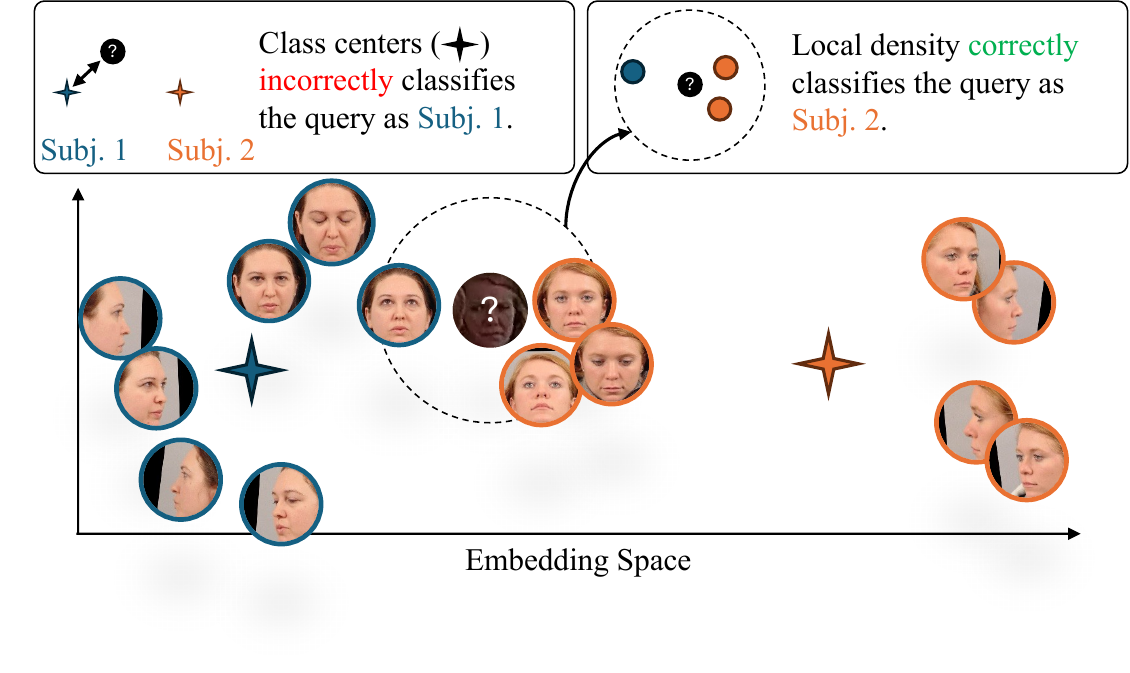}
    \vspace{-0.45in}
    \caption{Illustration of the effect of local density in face recognition. Each face represents a gallery feature, and distinct clusters correspond to different viewpoints. Conventional biometric systems only represent each subject by a single class center (shown as the stars). However, a probe sample may lie close to this center yet remain far from all actual clusters, leading to false acceptance in verification or open-set retrieval. In contrast, the local density of the gallery feature distribution captures the true intra-subject variability and correctly rejects such non-mated probes.}
    \label{fig:teaser}
\end{figure}

In real-world deployments, it is increasingly feasible and common to collect, process, and store multiple samples for each enrolled identity, making the multi-sample gallery setting valid and practically relevant. Therefore, modern large-scale biometric galleries tend to have multiple samples per subject~\cite{ross2006handbook}. For example, CCPG~\cite{li2023depth} includes $33$ videos per gallery subject on average.

Despite this, most biometric methods have not explicitly explored how to leverage multi-sample galleries. Existing approaches are often simplistic, relying on score averaging across gallery samples~\cite{poh2002multi} or assuming that each subject can be represented by a single aggregated feature embedding~\cite{su2024open,kim2022cluster}. These models extract biometric features from input samples (\eg, images or videos) and then aggregate them into a single class center per subject, overlooking the intra-subject variation in multi-sample galleries.

Na\"ively averaging gallery feature embeddings or scores fails to capture that subject’s true variability. As illustrated in \cref{fig:teaser} (top-left), such approaches rely solely on similarity to a subject's global center, ignoring the diverse conditions (\eg, pose, illumination, viewpoint) that produce distinct clusters within the same subject. Consequently, collapsing multi-cluster gallery distributions into a single center leads to suboptimal decision boundaries and degraded discrimination. This limitation is particularly detrimental in open-set scenarios, where accurately rejecting unseen subjects requires modeling the full extent of intra-subject variation.

To illustrate this issue, \cref{fig:teaser} presents a face recognition example in which exemplar subjects’ gallery embeddings form several distinct clusters associated with different camera viewpoints. A probe may lie near the center of subject \#1 yet remain relatively far from its clusters, leading to a false acceptance of a non-mated sample. Conversely, a genuine probe from an underrepresented
cluster may be erroneously rejected. These failure cases highlight the limitations of relying solely on class centers and motivate the need for methods that leverage the \emph{local density} of gallery features to more accurately reflect intra-subject variation.

We propose \paper, a novel scoring framework that models intra-subject variability through local structure rather than global averaging in biometric systems. The key idea is that the $k$-th nearest neighbor (\(k\)-NN) provides a simple yet effective way to capture the local geometry of a subject’s gallery features. As shown in \cref{fig:teaser}, a non-mated probe may lie near the global gallery center yet remain distant from all true clusters, causing traditional center-based scoring to falsely accept it. \paper addresses this by assessing how well a probe aligns with the local neighborhoods of genuine gallery samples using \(k\)-NN. Probes distant from their nearest neighbors are more likely imposters, while genuine probes tend to be closer. By leveraging the local gallery structure, \paper substantially improves verification accuracy and open-set robustness. Moreover, \paper is easily deployable, implemented in just a single line of code. It is a plug-and-play module compatible with any biometric model, independent of architecture or loss function, and introduces negligible computational overhead.
 
We evaluate \paper across three major biometric modalities—face recognition, gait recognition, and person reID using diverse model architectures, including ResNet~\cite{gu2022clothes}, ViT~\cite{liu2024distilling, ye2024biggait}, and Swin Transformer~\cite{fan2023exploring}. These models are trained with various loss functions such as cross-entropy variants~\cite{kim2022adaface} and triplet loss~\cite{hermans2017defense}. Across all settings, \paper consistently surpasses standard scoring methods, delivering substantial gains in both open-set performance (reducing FNIR@FPIR from \(53\%\) to \(40\%\)) and verification accuracy (raising TAR@FAR from \(51\%\) to \(74\%\)).

Additionally, we present comprehensive theoretical and empirical analyses to investigate the behavior of our proposed algorithm. 
Theoretically, we model the expected performance with and without \paper based on score statistics and provide guidance on the conditions under which \paper is expected to enhance open-set biometrics and verification.
Empirically, we validate the soundness of our theoretical insights on real-world datasets and through Monte Carlo simulations and examine the impact of \(k\) in \(k\)-NN and intra-subject variations on performance. 
We find that the algorithm consistently improves performance across a wide range of conditions and underscores its robustness.

We summarize our contributions as follows:
\begin{itemize} 
    \item We propose \paper, that incorporates the local gallery density information on gallery feature distribution around each probe into biometric scores.
    \item We validate the superior performance of our method across a variety of biometric models and tasks, significantly enhancing recognition accuracy.
    \item We present detailed analyses comparing traditional distance metrics and our algorithm in theory, through simulations on synthetic datasets, and on real-world datasets.
\end{itemize}

%% file: sec/2_related_work.tex
\section{Related Work}
\label{sec:related_work}

\paragraph{Biometrics.}
Biometric recognition has been extensively studied across modalities such as face recognition~\cite{ebrahimi2025gif, you2025lvface, peng2025stylized, kim2025vigface, kim2025idface}, gait recognition~\cite{shen2025lidargait++, yang2025bridging, jin2025denoising, ye2025biggergait, habib2025cargait, wang2025gait, huang2025learning}, person reID~\cite{wang2025secap, cui2025dkc, liang2025differ, yuan2025poses, pang2025identity, chen2025learning, khalid2025bridging, li2025one, zhang2025viperson, su2025hamobe, xu2025self}, and combination of modalities~\cite{liu2024farsight, liu2025person, zhu2025quality, kim2025sapiensid}.
Most prior work focuses on designing network architectures or loss functions to obtain more discriminative representations. However, scoring typically relies on the subjects’ aggregated gallery centers, overlooking the underlying structure of the gallery features. In contrast, we introduce \paper, which leverages the local density of gallery features to produce more reliable similarity scores. Also, while many existing methods emphasize verification or closed-set retrieval, only a few~\cite{gunther2017toward, su2024open} explicitly address the more challenging open-set setting.

\paragraph{Open-Set Recognition.}
Open-set recognition~\cite{wang2022openauc, kuchibhotla2022unseen, vaze2022openset, kong2021opengan, bendale2016towards, cruz2024operational, ren2023chatgpt} aims to distinguish between known and unknown classes in image classification. As noted in~\cite{su2024open}, the known classes correspond to those in the training set, while the unknown classes are those not used for training~\cite{geng2020recent}. In contrast, even in closed-set biometrics, the test subjects are not used for training, and mated probes are test subjects enrolled in the gallery~\cite{jain2021biometrics}. Open-set biometrics differs from closed-set biometrics in that non-mated subjects do not appear in either the training set or the gallery~\cite{gunther2017toward, grother2019face}.

While previous studies have explored the use of \(k\)-NN for out-of-distribution (OOD) detection~\cite{sun2022out} and open-set recognition~\cite{yang2022openood}, \(k\)-NN scores alone do not have sufficient discriminative information for biometric recognition. In contrast, our approach leverages \(k\)-NN to model the local feature density around probes and integrates this density measure into biometric scoring. 

\paragraph{Re-ranking.}
Re-ranking~\cite{zhong2017re, bai2019re, tan2021instance, sarfraz2018pose, shen2021re, ouyang2021contextual, yang2025cheb, zhang2023graph, barman2017shape} enhances the original ranking of probes obtained from a pretrained person reID model. 
Notably, \cite{zhong2017re} improves performance by considering if the probe is among the closest neighbors of the galleries. 
Although re-ranking can be applied to various biometric modalities, it depends on a large set of probes or queries to function effectively. 
However, in real-world applications, probes are often received continuously, and their distribution is unknown. 
As a result, the theoretical foundation and performance of re-ranking suffer when only one probe is available. 
In contrast, our method processes each probe independently, without requiring other probes. 
Further, our method is more computationally efficient than re-ranking due to its simplicity.

%% file: sec/3_method.tex
\section{Method}
\label{sec:method}

\crefname{theorem}{Thm.}{Thms.}

\subsection{Problem Setup}

Consider a gallery with \(N\) subjects, each with multiple samples (\eg, images or videos). 
For the \(i\)-th subject, the biometric features are 
\(\mathbf{g}_i^1, \ldots, \mathbf{g}_i^{m_i} \in \mathbb{R}^d\), 
and the full gallery is 
\(\mathcal{G} = \{\mathbf{g}_{i}^{j} \mid 1\!\le\! i\!\le\! N, 1\!\le\! j\!\le\! m_i\}\).
Given a query \(\mathbf{q} \in \mathbb{R}^d\), the goal is to compute
\(
\mathbf{s} = f(\mathbf{q}, \mathcal{G}),
\)
where each \(\mathbf{s}_i\) measures similarity between \(\mathbf{q}\) and subject \(i\).
The design of \(f\) directly influences verification and open-set retrieval performance.

A common strategy~\cite{su2024open} represents each subject with a single aggregated feature:
\begin{align}
    \mathbf{s}_i = 
    \left\langle 
    \frac{\mathbf{q}}{\|\mathbf{q}\|},
    \frac{1}{m_i}\!\sum_{j=1}^{m_i}\! \frac{\mathbf{g}_i^j}{\|\mathbf{g}_i^j\|}
    \right\rangle.
\end{align}
Methods such as EVM~\cite{gunther2017toward} and CAFace~\cite{kim2022cluster} follow this principle, differing only in aggregation details. 
However, all collapse intra-subject variation into a single prototype, discarding the feature distribution’s structure.

Alternatively, one can average per-sample similarities:
\begin{align}
    \mathbf{s}_i' = 
    \frac{1}{m_i}\!\sum_{j=1}^{m_i}\!
    \left\langle 
    \frac{\mathbf{q}}{\|\mathbf{q}\|},  
    \frac{\mathbf{g}_i^j}{\|\mathbf{g}_i^j\|}
    \right\rangle,
\end{align}
which retains all samples but treats them equally, ignoring cross-sample correlations or local substructure.

Existing scoring functions compute \(f\) independently for each subject, ignoring the gallery distribution. 
Specifically, conventional similarities assume that a subject’s features form a single cohesive cluster, whereas real-world data often contain multiple distinct regions corresponding to different viewpoints, illuminations, or other conditions. 
As a result, probes near sparse areas may appear close to the global mean but misaligned with any true region, leading to false acceptances or rejections. 
Neglecting this local structure limits the reliability of traditional scoring functions.

\subsection{\paper}

\begin{figure}
\renewcommand{\figurename}{Algorithm}

    \centering
    \begin{lstlisting}[style=modern, language=Python]
def local_score(model, x, G_S, G_M, k):
    """
    Parameters:
    - model: Pretrained biometric model
    - x: Probe input
    - G_S: (N, D) Per-subject gallery database
    - G_M: (N_m, D) Per-media gallery database
    - k: The "k" in KNN

    Returns:
    - (D,) updated score vector
    """

    # (N_g, ); similarity scores with G_S
    sim_scores = sim(model(x), G_S)

    # (1,); the k-th nearest neighbor in G_M
    knn_score = sim(model(x), G_M).sort()[:, -k]

    # (N_g, ); output similarity scores
    sim_scores += knn_score * (sim_scores == sim_scores.max())

    return sim_scores
\end{lstlisting}
    \vspace{-0.9em}
    \captionsetup{labelformat=empty} 
    \caption{Algorithm \thealgo. Python-style pseudo-code for \paper.}
    \label{alg:main_alg} 
    
\end{figure}
\setcounter{figure}{1}

\begin{figure*}
    \centering
    \includegraphics[width=0.9\linewidth]{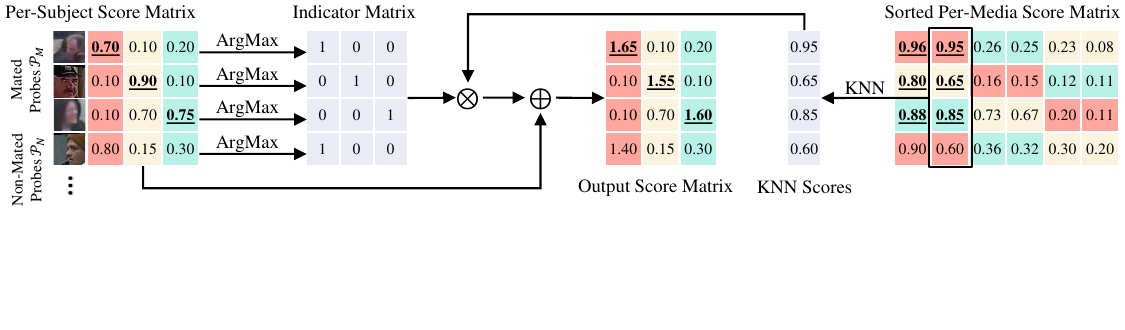}
    \vspace{-0.7in}
    \caption{A toy example of \paper, where genuine scores are bolded and underscored. Different colors represent different gallery subjects. The \(k\)-NN scores are calculated from the per-media score matrix for each probe and added to the maximum per-subject score to generate the output score matrix. For instance, for the first probe, a \(k\)-NN score of $0.95$ is added to the highest per-subject score of $0.7$, while all other scores remain unchanged. In the original score matrix, only one genuine score exceeds the maximum-per-probe non-mated score of $0.8$. However, after applying \paper, all genuine scores surpass the updated maximum-per-probe non-mated score of $1.4$.}
    \label{fig:toy_example}
\end{figure*}

Since gallery features naturally reflect intra-subject variation, we incorporate all samples in \(\mathcal{G}\) during scoring to capture the feature patterns within each subject.
Instead of averaging gallery features or pairwise similarities across all \(m_i\) samples, we focus on the local geometry of each subject’s feature space. 
We find the $k$-th nearest neighbor of the probe \(q\) among the gallery samples and calculate
\begin{align}
    {s}^{(k)} = 
    \operatorname*{top\text{-}k}_{i,j}
    \left(
    \left\langle 
    \frac{\mathbf{q}}{\left\lVert \mathbf{q} \right\rVert},
    \frac{\mathbf{g}_i^j}{\left\lVert \mathbf{g}_i^j \right\rVert}
    \right\rangle
    \right),
\end{align}
where \(\operatorname*{top\text{-}k}(\cdot)\) returns the $k$-th largest similarity and \(\mathbf{g}_i^j\) denotes the \(j\)-th feature of subject \(i\). 
This measures how closely the probe aligns with the most relevant local region of the gallery distribution. 
Genuine probes tend to yield larger \({s}^{(k)}\) as they cluster within dense regions of the correct subject’s gallery, while imposters tend to lie outside these modes and thus score lower.

To incorporate local geometry while preserving identity information, we fuse the $k$-NN score with the standard per-subject score \(\mathbf{s}\).
A na\"{i}ve averaging between \(\mathbf{s}_i\) and \(\mathbf{s}^{(k)}\) degrades performance because \(\mathbf{s}^{(k)}\) lacks explicit identity supervision, inflating non-match scores and narrowing the genuine–imposter margin.
For instance, we find that this na\"{i}ve fusion strategy reduces TAR@0.1\%FAR of KPRPE~\cite{kim2024keypoint} on the BRIAR dataset~\cite{cornett2023expanding} from 60\% to 30\%. 
Instead, we selectively apply the $k$-NN adjustment only to the top-scoring subject for each probe.

For instance, as illustrated by the third mated probe in \cref{fig:toy_example}, the similarity to the first subject after na\"{i}ve mean fusion is $(0.70 + 0.85)/2 = 0.775$. This value equals the similarity between the second probe and its true match (the second gallery subject), $(0.90 + 0.65)/2 = 0.775$. In this example, the first score is a non-match score, while the second is a match score. Thus, the verification performance suffers because the gap between match and non-match scores vanishes, causing fewer probes to be counted as true positives. Therefore, a more sophisticated algorithm is necessary.

To resolve this issue, let \(\mathbf{i} \in \{0,1\}^N\) be an indicator vector where \(\mathbf{i}_i = 1\) if \(\mathbf{s}_i = \max_j \mathbf{s}_j\) and \(0\) otherwise. 
The final fused scores \(\mathbf{t} \in \mathbb{R}^N\) are computed as:
\begin{align}
    \mathbf{t} = \mathbf{s} + \mathbf{i} \cdot \mathbf{s}^{(k)}.
\end{align}
This selective fusion increases the similarity only for the most probable subject, thereby avoiding global score inflation. \paper is presented in Alg.~\hyperref[alg:main_alg]{1}.
Note that it only capitalizes on the local gallery feature distribution and does not use the label of the gallery samples.

Compared to the na\"{i}ve fusion strategy, \paper increases the similarity scores of only the highest-scoring subject for each probe to avoid the broad inflation of non-match scores. In the toy example shown in \cref{fig:toy_example}, this selective increment keeps the similarity between the third probe and the second gallery subject at \(0.70\), which is considerably lower than any true match score. As a result, the non-match score does not overlap with the genuine matches. By not inflating the scores of non-top-ranked subjects, \paper preserves a wider separation between impostor and genuine distributions, improving verification performance.

\paper is training-free, efficient, and fully compatible with existing biometric models. 
Instead of assuming that a subject’s gallery features form a single cohesive cluster, it models their local structure, enabling improved verification and open-set recognition without any changes to network architecture or training objectives.

\subsection{Gallery Clustering}

While incorporating all gallery samples enhances robustness by capturing diverse intra-subject variations, it also increases memory and computation for subjects with many samples. To mitigate this, we optionally cluster each subject’s gallery features into a compact set of representative prototypes, preserving local structure while reducing cost.

For the \(i\)-th subject with gallery features 
\(\{\mathbf{g}_i^1, \ldots, \mathbf{g}_i^{m_i}\}\), 
we cluster them into \(C_i \ll m_i\) centers 
\(\{\mathbf{c}_i^1, \ldots, \mathbf{c}_i^{C_i}\}\), where each center is
\begin{align}
    \mathbf{c}_i^k = \frac{1}{|\mathcal{C}_i^k|} 
    \sum_{\mathbf{g}_i^j \in \mathcal{C}_i^k} \mathbf{g}_i^j,
\end{align}
and \(\mathcal{C}_i^k\) is the set of features in the \(k\)-th cluster. 
The resulting compact gallery is 
\(\tilde{\mathcal{G}} = \{\mathbf{c}_i^k\}_{i,k}\).

During inference, all similarity computations, including $k$-NN search and fusion, operate on cluster centers. The per-subject $k$-NN score becomes
\begin{align}
    \mathbf{s}_i^{(k)} = 
    \operatorname*{top\text{-}k}_{j}
    \left(
    \left\langle 
    \frac{\mathbf{q}}{\left\lVert \mathbf{q} \right\rVert},
    \frac{\mathbf{c}_i^j}{\left\lVert \mathbf{c}_i^j \right\rVert}
    \right\rangle
    \right),
\end{align}
where each \(\mathbf{c}_i^j\) captures a local mode of subject \(i\)’s distribution. This compact representation replaces redundant samples with informative prototypes, achieving a favorable balance between accuracy and efficiency.

\subsection{{Theoretical Analyses}}

We establish a theoretical framework to define the boundaries of applicability for \paper to provide bounds for scenarios where it enhances verification and open-set search. Specifically, \cref{thm:main} offers a criterion to assess whether the \(k\)-NN scores provide sufficient separation between mated and non-mated scores, resulting in improvements in both FNIR@FPIR and TAR@FAR.

Suppose the per-subject genuine scores are normally distributed as $\mathcal{N}(\mu_1, \sigma_1)$.
The per-subject imposter scores are distributed as $\mathcal{N}(\mu_2, \sigma_2)$. 
The \(k\)-NN scores for mated probes are distributed as $\mathcal{N}(\mu_3, \sigma_3)$, and the \(k\)-NN scores for non-mated probes are distributed as $\mathcal{N}(\mu_4, \sigma_4)$. 
Also, let $N_1$ be the number of mated probes, $N_2$ be the number of non-mated probes, $M$ be the number of gallery subjects, \(N_3\) be the total number of imposter scores, and $r_1$ and \(r_2\) be the tolerable number of maximum-per-probe non-mated scores (\eg, \(1\% \times N_2\)) and non-match scores (\eg, $0.1\% \times N_3$).

\begin{theorem}
    \label{thm:main}
    
    Alg.~\hyperref[alg:main_alg]{1} improves the expected open-set FNIR@FPIR at the expected threshold if
    \begin{align}
        \frac{\mu_2 + \mu_4 + \delta \sqrt{\sigma_2^2 + \sigma_4^2} - (\mu_1 + \mu_3)}{\sqrt{\sigma_1^2 + \sigma_3^2}} < \frac{(\mu_2 + \delta \sigma_2) - \mu_1}{\sigma_1}, \label{eq:open_set_theory}
    \end{align}
    where \(\delta = -\ln (-\ln(1 - r_1 / N_2))\) and improves verification TAR@FAR at the expected threshold if
        \begin{multline} \frac{\mu_2 + \mu_4 + \Phi^{-1}\left(\frac{N_3 - r_2 - \alpha}{N_3 - 2\alpha + 1}\right) \sqrt{\sigma_2^2 + \sigma_4^2} - (\mu_1 + \mu_3)}{\sqrt{\sigma_1^2 + \sigma_3^2}} \\ < \frac{\mu_2 + \Phi^{-1}\left(\frac{N_3 - r_2 - \alpha}{N_3 - 2\alpha + 1}\right) \sigma_2 - \mu_1}{\sigma_1}, \label{eq:verification_theory}  \end{multline}
    where $\Phi^{-1}(\cdot)$ is the inverse of the cumulative distribution function of normal distribution, and $\alpha=\pi/8$ is a constant.
\end{theorem}

For a proof of \cref{thm:main}, please refer to the supplementary.

%% file: sec/4_experiments.tex
\section{Experiments}
\label{sec:experiments}

\paragraph{Evaluation Protocols.}
We follow the evaluation protocols (train \vs test, gallery \vs probe, and mated \vs non-mated splits) of~\cite{su2024open} whenever possible. For newer datasets, we randomly generate evaluation protocols in the same way.

\paragraph{Evaluation Metrics.}
We evaluate biometric models from two aspects: verification and open-set retrieval. For verification, we report TAR@FAR at both \(0.1\%\) and \(1\%\) thresholds. 
For open-set search, we present the FNIR@FPIR at \(0.1\%\) and \(1\%\) thresholds for larger datasets and \(1\%\) and \(5\%\) for smaller ones. Given the inherent variability in FNIR@FPIR~\cite{su2024open}, we provide their means and 95\% confidence intervals from the results of \(50\) runs. 

\paragraph{Datasets.}
In addition to IJB-S, Gait3D, and CCVID~\cite{su2024open}, we evaluate on CCPG~\cite{li2023depth} for RGB-based gait recognition and BRIAR~\cite{cornett2023expanding} for multi-modal fusion.
CCPG targets clothing-changing gait recognition and person reID. It contains \num{100} training and \num{100} test subjects with \num{8388} and \num{8178} sequences, respectively. 
CCPG provides RGB images with faces and shoes grayed out, enabling RGB-based gait recognition methods like BigGait~\cite{ye2024biggait}, which outperform silhouette-based approaches.
BRIAR focuses on long-range and UAV-based biometrics across modalities including face, gait, and person reID. 
Its training set has \num{995} subjects, and we follow official protocol v6.1.0 with \num{1280} gallery subjects and \num{6813} probe sequences to evaluate \paper’s performance in whole-body human recognition.

\subsection{Ablation Studies}

\paragraph{Effect of \(k\).}
We analyze the sensitivity of verification and open-set recognition performance to \( k \) across three modalities: face recognition on BRIAR~\cite{cornett2023expanding} with IJB-S~\cite{kalka2018ijb}, gait recognition on Gait3D~\cite{zheng2022gait} with SwinGait~\cite{fan2023exploring}, and person reID on CCVID with CAL~\cite{gu2022clothes}. 
\cref{fig:ablation} illustrates the summed verification and open-set search performance. 
The open-set performance remains robust across different \(k\) values, indicating that \paper generalizes well. For verification, the performance degrades as $k$ increases. 
In general, smaller \( k \) values yield better performance, particularly for datasets with smaller gallery sizes (\eg, Gait3D with only 5.4 sequences per subject on average). 
Larger values of \( k \) tend to degrade verification performance due to an increased chance of including neighbors of incorrect gallery identities for genuine probes.
Thus, we choose \(k\)=1 in our experiments.

\begin{figure}
\centering
\begin{minipage}[t]{.47\linewidth}
  \centering
  \includegraphics[width=\linewidth]{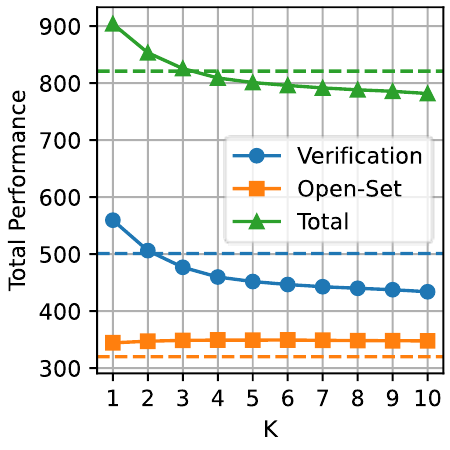}
  \captionof{figure}{The effect of \(k\) on the overall performance. The total performance is the sum of TPIR@FPIR (\ie, \(1 - \)FNIR@FPIR), TAR@FAR, and rank-1 accuracy. Dashed lines: without \paper.}
  \label{fig:ablation}
\end{minipage}\hfill%
\begin{minipage}[t]{.47\linewidth}
  \centering
  \includegraphics[width=\linewidth]{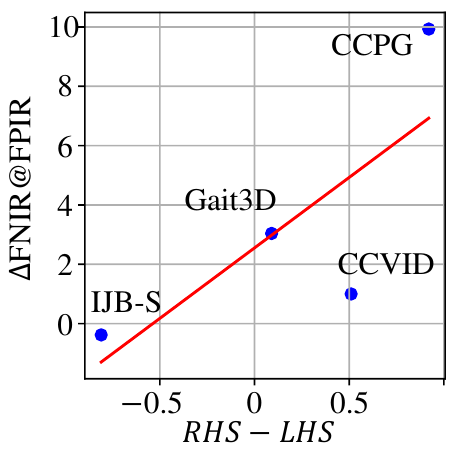}
  \captionof{figure}{The relationship between performance gains from \paper on various datasets and the theoretical difference between the right-hand and the left-hand side of \cref{eq:open_set_theory}. Red line: linear regression.}
  \label{fig:knn_perf_w_theory}
\end{minipage}
\end{figure}

\paragraph{Effects of Gallery Clustering.}
To examine the trade-off between storage and performance, we evaluate clustering of each subject’s gallery prior to scoring. Specifically, the gallery is clustered into \(C_i\) clusters using $k$-means, and each cluster is represented by its centroid. 

As shown in \cref{tab:gallery_clustering}, open-set FNIR@FPIR improves as \(C_i\) increases. This indicates that broader coverage of a subject’s  intra-subject variation reduces false rejections in open-set scenarios. In contrast, verification (TAR@FAR) remains stable across cluster sizes, suggesting that only a few centroids suffice for genuine match identification.

Sensitivity to clustering varies by dataset. Gait3D~\cite{zheng2022gait} has few, highly variable samples per subject (avg.\ $\approx5$), so clustering offers limited gains. In contrast, CCPG~\cite{li2023depth}, with more samples per subject (avg.\ $\approx35$) under controlled conditions, achieves near-full performance using under 50\% of samples. Thus, redundant gallery features can be compressed with minimal accuracy loss.

In subsequent experiments, we use the full gallery to report the best performance.

\subsection{Results}

\paragraph{Comparison with SoTA.}
We evaluate \paper using SoTA models across diverse biometric modalities, datasets, and gallery aggregation methods, including CAFace~\cite{kim2022cluster} for face and EVM~\cite{gunther2017toward}, which fits Weibull distributions to gallery features. As shown in \cref{tab:all_results}, \paper consistently improves over the original scores across all settings. On CCPG, for instance, it enhances BigGait from \(82.46\%\) to \(92.95\%\) (by \(10.49\%\)) on TAR@0.1\%FAR and from \(46.95\%\) to \(37.04\) (by \(9.93\%\)) on FNIR@0.1\%FPIR, without degrading performance for any model. When combined with EVM, \paper further boosts verification accuracy, \eg, improving TAR@0.1\%FAR on Gait3D from \(69.68\%\) to \(79.84\%\) (by \(10.16\%\)), while maintaining low inference cost and requiring only a single hyperparameter.

\begin{table}
    \centering\small
    \resizebox{\linewidth}{!}{
    \begin{NiceTabular}{lcccccccc}
        \toprule
        \Block{2-1}{Model} & \Block{2-1}{Dataset} & \Block{2-1}{\(C_i\)} & \Block{2-1}{\% Gal.} & \Block{2-1}{Run-\\time} & \Block{1-2}{FNIR@FPIR $\downarrow$} & & \Block{1-2}{TAR@FAR $\uparrow$} & \\ \cmidrule(lr){6-7} \cmidrule(lr){8-9}
        & & & & & 0.1\% & 1\% & 0.1\% & 1\% \\ \midrule 
        \Block{4-1}{SwinGait \\ \cite{fan2023exploring}} & \Block{4-1}{Gait3D\textsuperscript{\textdagger}} & N/A & \(<1\) & \(124\) & \(81.84\) & \(66.09\) & \(57.81\) & \(78.95\) \\
        & & \(5\) & \(53.7\) & \(246\) & \(83.50\) & \(69.10\) & \(\mathbf{63.65}\) & \({78.95}\) \\
        & & \(20\) & \(92.3\) & \(319\) & \({81.01}\) & \({65.43}\) & \(\mathbf{63.65}\) & \(\mathbf{79.08}\) \\
        & & \(\infty\) & \(100\) & \(338\) & \(\mathbf{78.80}\) & \(\mathbf{62.98}\) & \(\mathbf{63.65}\) & \(\mathbf{79.08}\) \\ \midrule
        \Block{4-1}{BigGait \\ \cite{ye2024biggait}} & \Block{4-1}{CCPG \\ \cite{li2023depth}} & N/A & \(<1\) & \(29\) & \(46.95\) & \(34.53\) & \(82.46\) & \(95.64\) \\
        & & \(5\) & \(14.2\) & \(60\) & \({41.04}\) & \({29.52}\) & \({92.25}\) & \({97.00}\) \\
        & & \(20\) & \(56.8\) & \(175\) & \({39.23}\) & \({27.34}\) & \({92.89}\) & \({97.00}\) \\
        & & \(\infty\) & \(100\) & \(301\) & \(\mathbf{37.04}\) & \(\mathbf{25.34}\) & \(\mathbf{92.95}\) & \(\mathbf{97.01}\) \\
        \bottomrule
    \end{NiceTabular}
    }
    \caption{Open-set FNIR@FPIR and verification TAR@FAR under different \(C_i\).
    \% Gal.: the ratio of the size of the clustered gallery to that of the full gallery.
    Runtime is reported in milliseconds (ms).
    \textsuperscript{\textdagger}FNIR@FPIR reported at 1\% and 5\% thresholds.}
    \label{tab:gallery_clustering}
\end{table}

\begin{table*}
    \centering\small
    \resizebox{0.75\linewidth}{!}{
    \begin{NiceTabular}{lcccccccc}
        \toprule
        \Block{2-1}{Modality} & \Block{2-1}{Model} & \Block{2-1}{Dataset} & \Block{2-1}{Gallery} & \Block{2-1}{Scoring} & \Block{1-2}{FNIR@FPIR $\downarrow$} & & \Block{1-2}{TAR@FAR $\uparrow$} \\ \cmidrule(lr){6-7} \cmidrule(lr){8-9} 
        &  & & & & 1\% & 5\% & 0.1\% & 1\% \\ \midrule 
        \Block{6-1}{Face} & \Block{6-1}{KPRPE \\ \cite{kim2024keypoint}} & \Block{6-1}{IJB-S \\ \cite{kalka2018ijb}} & \Block{3-1}{Avg} & N/A & \(\mathbf{94.11}\) & \({88.64}\) & \(10.69\) & \(25.47\) \\
        & & & & Ours* & \(94.17\) & \(\mathbf{88.00}\) & \(11.33\) & \(25.91\) \\
        & & & & Ours & \(94.49\) & \(89.12\) & \(\mathbf{21.76}\) & \(\mathbf{45.47}\) \\ \cmidrule{4-9}
        & & & \Block{3-1}{EVM \\ \cite{gunther2017toward}} & N/A & \(\mathbf{88.20}\) & \(\mathbf{83.43}\) & \(18.93\) & \(31.83\) \\
        & & & & Ours* & \({94.35}\) & \({83.94}\) & \({17.76}\) & \({31.28}\) \\
        & & & & Ours & \({94.08}\) & \({84.63}\) & \(\mathbf{22.93}\) & \(\mathbf{36.83}\) \\
        \midrule
        \Block{12-1}{Gait} & \Block{6-1}{SwinGait \\ \cite{fan2023exploring}} & \Block{6-1}{Gait3D \\ \cite{zheng2022gait}} & \Block{3-1}{Avg} & N/A & \(81.84 \smallpm 1.76\) & \(66.09 \smallpm 0.69\) & \(57.81\) & \(78.95\) \\
        & & & & Ours* & \(\mathbf{78.59 \smallpm 2.04}\) & \(\mathbf{61.76 \smallpm 0.94}\) & \(\mathbf{65.62}\) & \(\mathbf{86.47}\) \\
        & & & & Ours & \({78.80 \smallpm 1.99}\) & \({62.98 \smallpm 0.78}\) & \({63.65}\) & \({79.08}\) \\ \cmidrule{4-9}
        & & & \Block{3-1}{EVM \\ \cite{gunther2017toward}} & N/A & \(83.63 \smallpm 1.54\) & \(62.40 \smallpm 1.26\) & \(69.68\) & \(88.61\) \\
        & & & & Ours* & \(\mathbf{83.36 \smallpm 1.55}\) & \(\mathbf{61.95 \smallpm 1.23}\) & \({69.89}\) & \(\mathbf{92.20}\) \\
        & & & & Ours & \({83.38 \smallpm 1.53}\) & \({62.19 \smallpm 1.21}\) & \(\mathbf{79.84}\) & \({88.72}\) \\
        \cmidrule{2-9}
        & \Block{6-1}{BigGait \\ \cite{ye2024biggait}} & \Block{6-1}{CCPG\textsuperscript{\textdagger} \\ \cite{li2023depth}} & \Block{3-1}{Avg} & N/A & \(46.95 \smallpm 2.02\) & \(34.53 \smallpm 1.96\) & \(82.46\) & \(95.64\) \\
        & & & & Ours* & \(\mathbf{36.39 \smallpm 2.77}\) & \(\mathbf{24.10 \smallpm 2.12}\) & 87.86 & 97.04 \\
        & & & & Ours & \({37.04 \smallpm 2.71}\) & \({25.34 \smallpm 1.99}\) & \(\mathbf{92.95}\) & \(\mathbf{97.01}\) \\ \cmidrule{4-9}
        & & & \Block{3-1}{EVM \\ \cite{gunther2017toward}} & N/A &  \(46.64 \smallpm 3.93\) & \(34.94 \smallpm 3.41\) & \(86.25\) & \(97.37\) \\
        & & & & Ours* & \(\mathbf{41.93 \smallpm 4.14}\) & \(\mathbf{30.23 \smallpm 3.23}\) & \({87.13}\) & \({97.10}\) \\
        & & & & Ours & \({41.99 \smallpm 4.13}\) & \({30.41 \smallpm 3.21}\) & \(\mathbf{93.29}\) & \(\mathbf{97.78}\) \\ 
        \midrule
        \Block{6-1}{ReID} &\Block{6-1}{CAL \\ \cite{gu2022clothes}} & \Block{5-1}{CCVID \\ \cite{gu2022clothes}} &  \Block{3-1}{Avg} & N/A & \(39.98 \smallpm 0.88\) & \(34.38 \smallpm 0.91\) & \(68.19\) & \(81.81\) \\
        & & & & Ours* & \(\mathbf{38.66 \smallpm 0.85}\) & \(\mathbf{33.31 \smallpm 0.93}\) & 69.24 & 82.97 \\
        & & & & Ours & \({38.98 \smallpm 0.77}\) & \(34.19 \smallpm 0.84\) & \(\mathbf{75.23}\) & \(\mathbf{85.34}\) \\ \cmidrule{4-9}
        & & & \Block{3-1}{EVM \\ \cite{gunther2017toward}} & N/A & \(37.48 \smallpm 2.32\) & \(\mathbf{28.72 \smallpm 1.05}\) & \(75.21\) & \(86.12\) \\
        & & & & Ours* & \(\mathbf{37.30 \smallpm 2.06}\) & \({30.04 \smallpm 0.88}\) & \({72.53}\) & \({82.99}\) \\
        & & & & Ours & \(37.48 \smallpm 1.99\) & \(30.98 \smallpm 0.82\) & \(\mathbf{78.00}\) & \(\mathbf{86.26}\) \\ 
        \bottomrule
    \end{NiceTabular}
    }
    \caption{Open-set FNIR@FPIR and verification TAR@FAR on real-world biometric datasets. 95\% confidence intervals are reported for FNIR@FPIR. Ours* denotes \paper under the 1:1 setting. \textsuperscript{\textdagger}FNIR@FPIR reported at 0.1\% and 1\% thresholds. }
    \label{tab:all_results}
\end{table*}

\paragraph{Theoretical Predictions.}
\cref{thm:main} provides a theoretical tool to predict the real-world performance of \paper. To validate the accuracy of the theoretical predictions, we compare performance gains from \paper to the difference between the right-hand side (RHS) and the left-hand side (LHS) of \cref{eq:open_set_theory} in \cref{fig:knn_perf_w_theory}, averaged over multiple gallery sets. The strong correlation (Pearson coefficient 0.77) and large positive slope (4.75) confirm that increases in this theoretical difference align with real-world performance improvements, demonstrating \cref{thm:main}’s predictive power despite its assumptions on score distributions.

\begin{table}
    \centering\small
    \resizebox{0.9\linewidth}{!}{
    \begin{NiceTabular}{lcccccc}
        \toprule
        \Block{2-1}{Task} & \Block{2-1}{Model} & \Block{2-1}{Ours} & \Block{1-2}{FNIR@FPIR $\downarrow$} & & \Block{1-2}{TAR@FAR $\uparrow$} & \\ 
        \cmidrule(lr){4-5} \cmidrule(lr){6-7} 
        & & & CCVID & BRIAR & CCVID & BRIAR \\ \midrule 
        \Block{2-1}{Face} & \Block{2-1}{KPRPE \\ \cite{kim2024keypoint}} & \ding{55} & \(12.99\) & \(93.6\) & \(91.12\) & \(14.7\) \\
        & & \ding{51} & \(\mathbf{10.67}\) & \(\mathbf{92.9}\) & \(\mathbf{93.88}\) & \(\mathbf{17.3}\) \\ \midrule
        \Block{2-1}{Gait} & \Block{2-1}{BigGait \\ \cite{ye2024biggait}}  & \ding{55} & \(71.05\) & \(84.3\) & \(50.72\) & \({35.3}\) \\
        & & \ding{51} & \(\mathbf{59.92}\) & \(\mathbf{84.0}\) & \(\mathbf{73.98}\) & \(\mathbf{37.7}\) \\ \midrule
        \Block{2-1}{ReID} & \Block{2-1}{CLIP3D-\\ReID~ \cite{liu2024distilling}} & \ding{55} & \(52.77\) & \(\mathbf{86.8}\) & \(62.11\) & \(28.2\) \\
        & & \ding{51} & \(\mathbf{40.49}\) & \({87.2}\) & \(\mathbf{80.46}\) & \(\mathbf{32.2}\)\\ \midrule
        \Block{2-1}{Whole\\ Body} & \Block{2-1}{FarSight \\ \cite{liu2024farsight}} & \ding{55} & \(15.41\) & \(76.3\) & \(87.41\) & \(39.2\) \\
        & & \ding{51} & \(\mathbf{10.74}\) & \(\mathbf{73.4}\) & \(\mathbf{91.73}\) & \(\mathbf{53.5}\) \\
        \bottomrule
    \end{NiceTabular}
    }
    \caption{Open-set FNIR@FPIR, verification TAR@FAR on the BRIAR~\cite{cornett2023expanding} and CCVID~\cite{gu2022clothes} datasets.}
    \label{tab:briar_results}
\end{table}

\paragraph{1:1 Evaluation.}
Although standard deployments assume access to all enrolled subjects, certain practical scenarios restrict verification to a single available subject. To assess whether \paper remains effective under this constraint, we evaluate it in a 1:1 setting, where each decision involves a single probe and a single gallery subject. In this setup, our method enhances similarity scores using only the gallery samples of the active subject (see supplementary materials for an illustration). Results for this configuration are reported as Ours* in \cref{tab:all_results}.
Across modalities and datasets, \paper delivers strong verification and open-set performance even without access to the full gallery. Surprisingly, its open-set results often surpass those obtained when all subjects are available. For example, on CCPG~\cite{li2023depth} using BigGait~\cite{ye2024biggait}, \paper achieves \(36.29\%\) FNIR@0.1\%FPIR, compared to \(46.95\%\) without our method and \(37.04\%\) when using the entire gallery. For verification, while performance is generally lower than in the full-gallery setting, \paper still provides substantial gains over baseline scoring. Notably, on Gait3D~\cite{zheng2022gait} with SwinGait~\cite{fan2023exploring}, it improves TAR@0.1\%FAR from \(57.81\%\) to \(65.62\%\), whereas using the entire gallery yields only \(63.65\%\). These results demonstrate that \paper effectively models intra-subject variation and enables reliable verification even when only a single subject’s gallery is available.

\paragraph{Human Recognition.}
We evaluate \paper across four modalities in the multi-modal BRIAR benchmark: face, gait, person re-identification (reID), and whole-body recognition. Each modality uses its own feature extractor, and results are summarized in \cref{tab:briar_results}. \paper yields consistent gains in open-set performance, with especially strong improvements on CCVID. For instance, it reduces FNIR@1\%FPIR for reID from \(52.77\) to \(40.49\) and boosts TAR@0.1\%FAR for gait from \(50.72\%\) to \(73.98\%\). On the more challenging BRIAR~\cite{cornett2023expanding} dataset, \paper also improves performance in nearly all cases. These results demonstrate that \paper integrates seamlessly into real-world systems, enhancing not only individual models but also complex multi-modal human recognition pipelines.

\begin{table}
    \centering\small
    \resizebox{\linewidth}{!}{
    \begin{NiceTabular}{lccccccc}
        \toprule
        \Block{2-1}{Task} & \Block{2-1}{Model} & \Block{2-1}{Dataset} & \Block{2-1}{Method} & \Block{1-2}{FNIR@FPIR $\downarrow$} & & \Block{1-2}{TAR@FAR $\uparrow$} & \\ \cmidrule(lr){5-6} \cmidrule(lr){7-8} 
        & & & & 0.1\% & 1\% & 0.1\% & 1\% \\ \midrule 
        \Block{5-1}{Gait} & \Block{5-1}{BigGait \\ \cite{ye2024biggait}}  & \Block{5-1}{CCPG \\ \cite{li2023depth}} & Max & \(42.87\) & \(28.92\) & \(86.94\) & \(\mathbf{97.11}\) \\
        & & & Min & \(83.23\) & \(70.97\) & \(48.65\) & \(77.88\) \\
        & & & Mean & \(51.08\) & \(38.12\) & \(79.22\) & \(95.20\) \\
        & & & R-R & \(47.84\) & \(36.56\) & \(81.12\) & \(95.14\) \\
        & & & Ours & \(\mathbf{37.04}\) & \(\mathbf{25.34}\) & \(\mathbf{92.95}\) & \({97.01}\) \\ \midrule
        \Block{5-1}{ReID} & \Block{5-1}{CAL \\ \cite{gu2022clothes}} & \Block{5-1}{CCVID\textsuperscript{\textdagger} \\ \cite{gu2022clothes}} & Max & \(40.55\) & \(\mathbf{32.89}\) & \(69.65\) & \(82.83\) \\
        & & & Min & \(54.49\) & \(45.86\) & \(60.97\) & \(80.36\) \\
        & & & Mean & \(\mathbf{38.74}\) & \(34.38\) & \(69.16\) & \(82.87\) \\
        & & & R-R & \(40.52\) & \({34.12}\) & \(68.88\) & \(83.60\) \\
        & & & Ours & \({38.98}\) & \({34.19}\) & \(\mathbf{75.23}\) & \(\mathbf{85.34}\) \\
        \bottomrule
    \end{NiceTabular}
    }
    \caption{Open-set FNIR@FPIR and verification TAR@FAR of multi-sample methods. \textsuperscript{\textdagger}FNIR@FPIR reported at 1\% and 5\% thresholds. [R-R: Re-ranking~\cite{zhong2017re}]}
    \label{tab:multi-sample}
\end{table}

\paragraph{Comparison with Multi-Sample Methods.}
We compare our scoring function against standard multi-sample aggregation strategies that collapse multiple gallery samples into a single score. Specifically, we evaluate three methods, mean, max, and min~\cite{best2014unconstrained}. As shown in \cref{tab:multi-sample}, these approaches exhibit a consistent ordering: our method delivers the strongest performance, followed by max, with mean and min trailing behind. For instance, on CCPG~\cite{li2023depth}, our approach achieves the lowest FNIR (\(37.04\%\)@0.1\%FPIR), compared to $42.87\%$ for max, $51.08\%$ for mean, and $83.23\%$ for min. A similar pattern appears in verification performance, where our TAR@0.1\%FAR reaches $92.95\%$, surpassing max ($86.94\%$), mean ($79.22\%$), and min ($48.65\%$). Overall, incorporating local geometric information into the scoring function yields more reliable verification under both open- and closed-set conditions than pooling-based methods that ignore intra-subject variation.

\paragraph{Comparison with Re-Ranking.}
For fair comparison, we evaluate \paper against a per-probe variant of the re-ranking algorithm~\cite{zhong2017re}, adapted to operate on each probe independently. Classical re-ranking leverages reciprocal neighbors—\ie, mutual nearest-neighbor relationships within a set of queries—and thus typically requires multiple probes to be effective. As shown in \cref{tab:multi-sample}, \paper consistently surpasses per-probe re-ranking by a substantial margin, underscoring its suitability for real-world biometric scenarios where probes often arrive individually.

%% file: sec/5_analyses.tex
\section{Monte Carlo Simulations}
\label{sec:analyses}

\subsection{Analysis on \cref{thm:main}}

We perform Monte Carlo simulations to assess the practical validity of \cref{thm:main} under both its theoretical assumptions and more general conditions. Specifically, we generate random score matrices consistent with \cref{thm:main}’s assumptions, where match, non-match, mated KNN, and non-mated scores follow normal distributions. We then compute \(\mu_3^*\), the value that equates the two sides of \cref{eq:open_set_theory}. Keeping all other distribution parameters fixed, we vary \(\mu_3\) and observe in \cref{fig:mc_sim_theorem_2_score_matrix_open_set} that performance decreases when \(\mu_3 < \mu_3^*\) and improves when \(\mu_3 > \mu_3^*\). This alignment between simulation and theory confirms \cref{thm:main}’s predictive validity for identifying when our method yields performance gains.

\begin{figure}
    \centering
    
    \begin{subfigure}{0.45\linewidth}
        \centering
        \includegraphics[width=\linewidth]{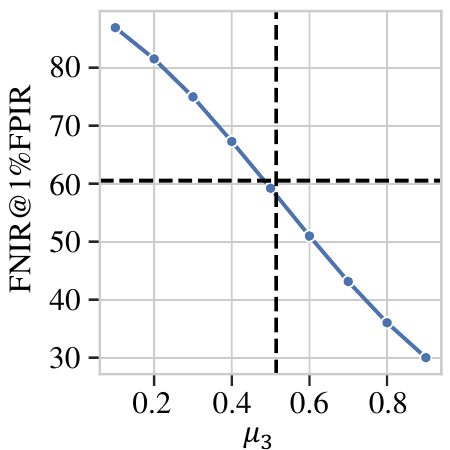}
        \caption{}
        \label{fig:mc_sim_theorem_2_score_matrix_open_set}
    \end{subfigure}
    \hfill
    \begin{subfigure}{0.45\linewidth}
        \centering
        \includegraphics[width=\linewidth]{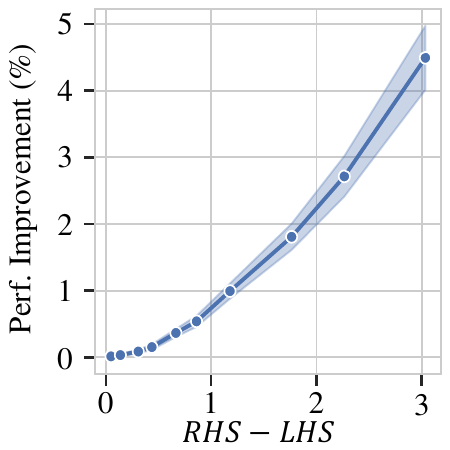}
        \caption{}
        \label{fig:mc_sim_theorem_2_features_open_set}
    \end{subfigure}

    \caption{(a) The effect of \(\mu_3\) on open-set FNIR@1\%FPIR. The vertical line indicates the theoretically predicted \(\mu_3^*\), and the horizontal line indicates the performance without our method. (b) The impact of the difference between the RHS and LHS of \cref{eq:open_set_theory} on the improvement of FNIR@FPIR by using our method. The shaded area is the \(99\%\) confidence interval.} 
    \label{fig:mc_sim_theorem_2_open_set}
\end{figure}

To account for deviations from our theoretical assumptions, we extend the simulations to randomly generated feature vectors and corresponding score matrices. We examine how performance gains relate to the numerical gap between the RHS and LHS of \cref{eq:open_set_theory}, where a larger gap indicates greater separation between mated and non-mated \(k\)-NN scores, and thus stronger performance improvements.

We create synthetic datasets containing both mated and non-mated samples to evaluate our method. Specifically, \(20\%\) of the \(N\) classes are designated as non-mated, containing only probe samples. For the \(i\)-th class, data are drawn from a multivariate Gaussian:
\begin{align}
    \mathcal{N}\!\left(\!\left[\cos\!\left(\tfrac{2\pi i}{N}\right), \sin\!\left(\tfrac{2\pi i}{N}\right)\right]^T, \sigma^2 \mathbf{I}\!\right),
\end{align}
yielding classes evenly distributed on the unit circle. Mated classes include \(M\) samples for both gallery and probe sets, while non-mated classes contain \(M\) probe samples only. Here, \(\sigma\) controls intra-class variation and \(N\) denotes the total number of classes. Each experiment is repeated \num{1000} times with different random seeds for robustness.

\cref{fig:mc_sim_theorem_2_features_open_set} demonstrates that \paper reliably separate mated from non-mated probes across varying \(\sigma\) and \(k\). The RHS consistently exceeds the LHS, which is consistent with observed open-set performance gains in most cases. Moreover, larger numerical differences correspond to greater improvements, indicating that despite its normality assumption, \cref{thm:main} generalizes well and effectively predicts when our method enhances FNIR@FPIR.

\subsection{Ablation Study}
\label{sec:analyses_empirical}

We ablate the hyperparameter \(k\) and the intra-class variation parameter \(\sigma\) using the same setup as in the previous subsection, reporting their effects on FNIR@1\%FPIR. Additional analyses on \cref{eq:verification_theory} and TAR@FAR are provided in the supplementary materials.

\begin{figure}
    \centering
    
    \begin{subfigure}{0.45\linewidth}
        \centering
        \includegraphics[width=\linewidth]{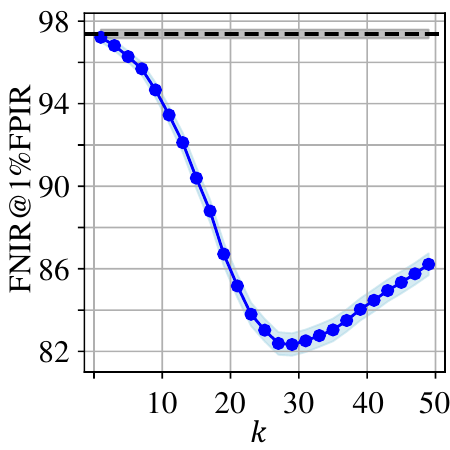}
        \caption{The effect of \(k\).}
        \label{fig:open_set_vs_k}
    \end{subfigure}
    \hfill
    \begin{subfigure}{0.45\linewidth}
        \centering
        \includegraphics[width=\linewidth]{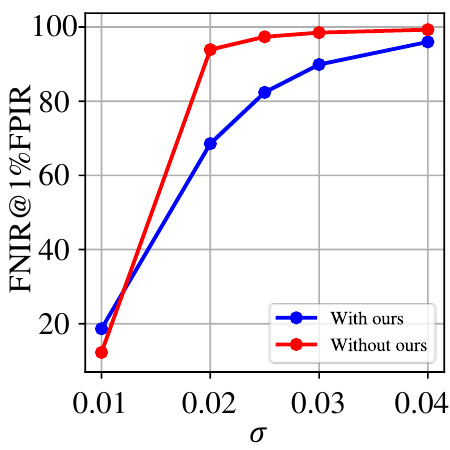}
        \caption{The effect of \(\sigma\).}
        \label{fig:open_set_vs_sigma_i}
    \end{subfigure}

    \caption{The effect of \(k\) and \(\sigma\) on open-set FNIR@1\%FPIR. In (a), the horizontal line denotes performance without our method. The shaded area is the \(99\%\) confidence interval.}
    \label{fig:synthetic_exp_open_set}
\end{figure}

\paragraph{Effect of \(k\).}
As shown in \cref{fig:open_set_vs_k}, open-set FNIR@1\%FPIR decreases sharply for small \(k\) and then rises gradually. This trend aligns with real-world observations, where \paper consistently outperforms the baseline. Compared to our simulations, the optimal \(k\) in practice is smaller—\(k = 1\) in our experiments. While the Gaussian assumption in simulation may not fully capture real-world feature distributions, our method consistently improves open-set performance, achieving over 15\% gains at the optimal \(k\).

\paragraph{Effect of \(\sigma\).}
As shown in \cref{fig:open_set_vs_sigma_i}, performance gains depend strongly on intra-class variation. For small \(\sigma\) (low variation), improvements are small, but as \(\sigma\) increases, \paper rapidly surpasses raw scores, with substantial gains in challenging cases. This pattern aligns with real datasets, where greater benefits occur in modalities with higher intra-subject variation such as gait and person reID. In face recognition, which exhibits lower variation, \paper still delivers notable improvements, particularly for verification.

%% file: sec/6_conclusion.tex
\section{Conclusion}
\label{sec:conclusion}

We present a simple yet effective algorithm to enhance open-set search and verification performance by leveraging intra-subject variation. We evaluate our method across various biometric tasks, utilizing models with different architectures and trained with diverse loss functions to demonstrate its efficacy. Further, we provide both empirical and theoretical analyses to understand the condition under which our method improves the performance.

\paragraph{Acknowledgements.}
This research is based upon work supported in part by the Office of the Director of National Intelligence (ODNI), Intelligence Advanced Research Projects Activity (IARPA), via 2022-21102100004. The views and conclusions contained herein are those of the authors and should not be interpreted as necessarily representing the official policies, either expressed or implied, of ODNI, IARPA, or the U.S. Government. The U.S. Government is authorized to reproduce and distribute reprints for governmental purposes notwithstanding any copyright annotation therein.

%% file: sec/X_suppl.tex
\clearpage
\setcounter{page}{1}
\maketitlesupplementary

\appendix

\section{Effect on Rank-1 Accuracy}

The effect of Alg.~\hyperref[alg:main_alg]{1} on rank-1 accuracy can be summarized in the following theorem:

\begin{theorem}
Let \( s_1, s_2, \ldots, s_n \) denote a set of similarity scores produced by a closed-set recognition algorithm, where \( s_i \) represents the similarity score between a probe and the \( i \)-th gallery subject. If Alg.~\hyperref[alg:main_alg]{1} adds a non-negative value \( d \) to the highest score, the closed-set accuracy at any rank does not change.
\label{thm:closed_set_rank}
\end{theorem}

\begin{proof}
    To show this, we can assume without loss of generality that the scores are in non-decreasing order, \ie, \( s_1 \le s_2 \le \dots \le s_n \). By Alg.~\hyperref[alg:main_alg]{1}, the highest score \( s_n \) is modified to \( s_n + d \), where \( d \ge 0 \). This results in the following inequality:
    \begin{align}
        s_1 \le s_2 \le \dots \le s_n \le s_n + d.
    \end{align}
    This adjustment preserves the relative ranking of the scores. 
    
    The closed-set accuracy at rank \( r \) is determined by the number of correct matches among the top \( r \) scores. Suppose the correct match corresponds to the score \( s_j \). If \( s_j \) is initially ranked within the top \( r \), we have \( j \ge n + 1 - r \). As the adjustment preserves relative ranking, \( s_j \) remains correctly ranked within the top \( r \).
    Conversely, if \( s_j \) is not among the top \( r \) scores initially, then \( j < n + 1 - r \). The increase in \( s_n \) does not affect the ranks of other scores below \( s_n \), ensuring \( s_j \) continues to be unranked within the top \( r\).

    In both cases, we conclude that Alg.~\hyperref[alg:main_alg]{1} does not change the closed-set accuracy at any rank. 
\end{proof}

\section{Proof of \cref{thm:main}}

In this section, we provide a proof of \cref{thm:main}. Before the proof, we reproduce the theorem below for your reference. 

\newcounter{savedeq}
\setcounter{savedeq}{\value{equation}}
\newcounter{savedtheorem}
\setcounter{savedtheorem}{\value{theorem}}

\setcounter{equation}{4}
\setcounter{theorem}{0}

Suppose the per-subject genuine scores are normally distributed as $\mathcal{N}(\mu_1, \sigma_1)$. 
The per-subject imposter scores are distributed as $\mathcal{N}(\mu_2, \sigma_2)$. 
The KNN scores for mated probes are distributed as $\mathcal{N}(\mu_3, \sigma_3)$, and the KNN scores for non-mated probes are distributed as $\mathcal{N}(\mu_4, \sigma_4)$.
Also, let $N_1$ be the number of mated probes, $N_2$ be the number of non-mated probes, $M$ is the number of gallery subjects, \(N_3\) be the total number of imposter scores, and $r_1$ and \(r_2\) be the tolerable number of maximum-per-probe non-mated scores (\eg, \(0.1\% \times N_2\)) and non-match scores (\eg, $0.1\% \times N_3$).

\begin{theorem}    
    Alg.~\hyperref[alg:main_alg]{1} improves the expected open-set FNIR@FPIR at the expected threshold if
    \begin{align}
        \frac{\mu_2 + \mu_4 + \delta \sqrt{\sigma_2^2 + \sigma_4^2} - (\mu_1 + \mu_3)}{\sqrt{\sigma_1^2 + \sigma_3^2}} < \frac{(\mu_2 + \delta \sigma_2) - \mu_1}{\sigma_1},
    \end{align}
    where \(\delta = -\ln (-\ln(1 - r_1 / N_2))\) and improves verification TAR@FAR at the expected threshold if
        \begin{multline} \frac{\mu_2 + \mu_4 + \Phi^{-1}\left(\frac{N_3 - r_2 - \alpha}{N_3 - 2\alpha + 1}\right) \sqrt{\sigma_2^2 + \sigma_4^2} - (\mu_1 + \mu_3)}{\sqrt{\sigma_1^2 + \sigma_3^2}} \\ < \frac{\mu_2 + \Phi^{-1}\left(\frac{N_3 - r_2 - \alpha}{N_3 - 2\alpha + 1}\right) \sigma_2 - \mu_1}{\sigma_1}, \end{multline}
    where $\Phi^{-1}(\cdot)$ is the inverse of the cumulative distribution function of normal distribution, and $\alpha=\pi/8$ is a constant.
\end{theorem}

\setcounter{theorem}{\value{savedtheorem}}
\setcounter{equation}{\value{savedeq}}

\begin{proof}
    We prove the two parts of the theorem, namely \cref{eq:open_set_theory} and \cref{eq:verification_theory} separately and we start with \cref{eq:open_set_theory}. 
    
    To prove \cref{eq:open_set_theory}, we begin by considering the FNIR threshold without KNN scores. By the Fisher-Tippett-Gnedenko theorem~\cite{gnedenko1943distribution, fisher1928limiting}, since the non-match scores follow a Gaussian distribution, the maximum non-mated scores converge to a Gumbel distribution. 
    
    Given the mean of non-match scores without KNN scores \(\mu_2\) and its standard deviation is \(\sigma_2\), the corresponding threshold is determined by the quantile of the Gumbel distribution. Specifically, the threshold for FPIR equal to \(\frac{r_1}{N_2}\) is given by: 
    \begin{align}
        \mu_2 + \left(-\ln -\ln \left(1 - \frac{r_1}{N_2}\right)\right) = \mu_2 + \delta \sigma_2,
    \end{align}
    if we let \(\delta = -\ln \left( -\ln \left(\frac{r_1}{N_2}\right) \right)\), which is the critical value corresponding to the desired FPIR. This threshold is derived from the inverse cumulative distribution function (CDF) of the Gumbel distribution. 

    Next, to compute the FNIR at this threshold, we use the fact that the match scores follow a Gaussian distribution with mean \(\mu_1\) and standard deviation \(\sigma_1\). The FNIR is then given by the complementary cumulative distribution function of the match scores at the threshold:
    \begin{align}
        \text{FNIR} = 1 - \Phi \left( \frac{(\mu_2 + \delta \sigma_2) - \mu_1}{\sigma_1} \right), \label{eq:expected_fnir}
    \end{align}
    where \(\Phi\) is the CDF of the standard normal distribution. This is the expected FNIR at the threshold based on the original match scores.

    Now, we introduce KNN scores. When KNN scores are used, the non-match scores become a sum of two Gaussian distributions: one with mean \(\mu_2\) and variance \(\sigma_2^2\), and the other with mean \(\mu_4\) and variance \(\sigma_4^2\). The resulting distribution is Gaussian with mean \(\mu_2 + \mu_4\) and standard deviation \(\sqrt{\sigma_2^2 + \sigma_4^2}\).

    Thus, the FPIR threshold with KNN scores is given by
    \begin{align}
        \mu_2 + \mu_4 + \delta \sqrt{\sigma_2^2 + \sigma_4^2}.
    \end{align}
    Similarly, the FNIR with KNN scores at the expected threshold is
    \begin{align}
        \text{FNIR}_{\text{KNN}} = 1 - \Phi \left( \frac{\mu_2 + \mu_4 + \delta \sqrt{\sigma_2^2 + \sigma_4^2} - (\mu_1 + \mu_3)}{\sqrt{\sigma_1^2 + \sigma_3^2}} \right), 
    \end{align}
    where \(\mu_1 + \mu_3\) is the mean of the match scores with KNN, and \(\sqrt{\sigma_1^2 + \sigma_3^2}\) is the standard deviation. Here it is assumed for simplicity that the KNN scores are added to all match scores, which is a reasonable assumption for real-world biometric systems.

    To prove the first part of the theorem, we show that Alg.~\hyperref[alg:main_alg]{1} improves the FNIR. This is equivalent to showing that the FNIR with KNN scores is smaller than the FNIR without KNN scores. To do so, we compare the arguments inside the CDFs of the two FNIR expressions. Specifically, we require:
    \begin{align}
        \frac{\mu_2 + \mu_4 + \delta \sqrt{\sigma_2^2 + \sigma_4^2} - (\mu_1 + \mu_3)}{\sqrt{\sigma_1^2 + \sigma_3^2}} < \frac{(\mu_2 + \delta \sigma_2) - \mu_1}{\sigma_1}.
    \end{align}
    This inequality guarantees that the FNIR with KNN scores is smaller (improved) compared to the FNIR without KNN scores. Since \(\Phi(\cdot)\) is a monotonically increasing function, the inequality implies that the FNIR with KNN scores is indeed lower.
    
    Next, we turn to the second part of the theorem: improving the verification TAR@FAR at the expected threshold. We begin by deriving the FAR threshold for the non-match scores, which follow a Gaussian distribution with mean \(\mu_2\) and standard deviation \(\sigma_2\). By \cite{elfving1947asymptotical}, the FAR threshold for a given false acceptance rate \( \frac{r_2}{N_3} \) is given by:
    \begin{align}
        \mu_2 + \Phi^{-1}\left(\frac{N_3 - r_2 - \alpha}{N_3 - 2\alpha + 1}\right) \sigma_2. \label{eq:expected_verify_threshold}
    \end{align}
    where \(\alpha = \frac{\pi}{8}\) is a constant and \(\Phi^{-1}\) is the inverse of the standard normal CDF.

    For verification TAR, we assume that Alg.~\hyperref[alg:main_alg]{1} adds the KNN scores to all non-match scores above this threshold. Thus, the new expected verification threshold becomes:
    \begin{align}
        \mu_2 + \mu_4 + \Phi^{-1}\left(\frac{N_3 - r_2 - \alpha}{N_3 - 2\alpha + 1}\right) \sqrt{\sigma_2^2 + \sigma_4^2}. \label{eq:expected_verify_threshold_with_knn}
    \end{align}

\begin{figure}
    \centering
    \includegraphics[width=\linewidth]{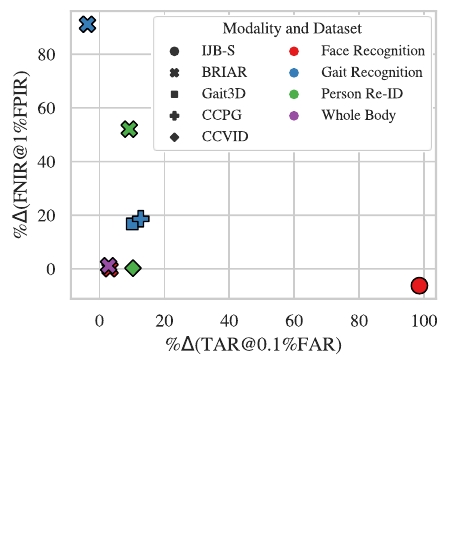}
    \vspace{-1.5in}
    \caption{The relative performance improvement (\(\%\Delta\)) of \paper on different modalities and datasets is calculated as \((\text{ours} - \text{baseline}) / \text{baseline}\), where ``ours'' represents TAR@0.1\%FAR or FNIR@1\%FPIR with \paper, and ``baseline'' refers to the corresponding values without \paper. A value of \((0, 0)\) indicates no improvement over the baseline in either metric. These results, derived from \cref{tab:all_results}, highlight the performance gains achieved by our approach.}
    \label{fig:relative_performance_diff}
\end{figure}

    Finally, we compare the z-scores of the verification thresholds with and without KNN scores. The z-score for the match score distribution (without KNN) is:
    \begin{align}
        \frac{\mu_2 + \mu_4 + \Phi^{-1}\left(\frac{N_3 - r_2 - \alpha}{N_3 - 2\alpha + 1}\right) \sqrt{\sigma_2^2 + \sigma_4^2} - (\mu_1 + \mu_3)}{\sqrt{\sigma_1^2 + \sigma_3^2}}.
    \end{align}
    For the case with KNN scores, the z-score is:
    \begin{align}
        \frac{\mu_2 + \Phi^{-1}\left(\frac{N_3 - r_2 - \alpha}{N_3 - 2\alpha + 1}\right) \sigma_2 - \mu_1}{\sigma_1},
    \end{align}
    By comparing the z-scores of the two distributions, we conclude that Alg.~\hyperref[alg:main_alg]{1} improves the verification TAR@FAR at the expected threshold if \cref{eq:verification_theory} holds.
\end{proof}

\section{Additional Experiments}

\begin{figure*}
    \centering
    \includegraphics[width=0.75\linewidth]{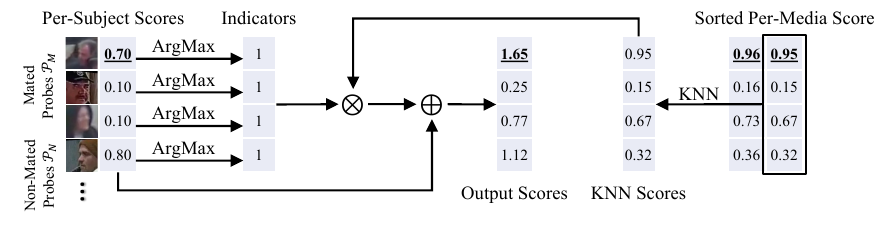}
    \caption{Illustration of single-subject evaluation using \paper. Compared to \cref{fig:toy_example}, when only one gallery subject is available, the indicators are always 1 since each probe has only one per-subject score. Therefore, the \(k\)-th largest score of the same subject is added to the per-subject scores.}
    \label{fig:single_subject}
\end{figure*}

\begin{table*}
    \centering\small
    \begin{NiceTabular}{lccccccccc}
        \toprule
        \Block{2-1}{Model} & \Block{2-1}{Dataset} & \Block{2-1}{Mod.} &
        \Block{2-1}{\(C_i\)} & \Block{2-1}{\% Gal.} & 
        \Block{2-1}{Run-\\time} &
        \Block{1-2}{FNIR@FPIR $\downarrow$} & &
        \Block{1-2}{TAR@FAR $\uparrow$} & \\ 
        \cmidrule(lr){7-8} \cmidrule(lr){9-10}
        & & & & & & 0.1\% & 1\% & 0.1\% & 1\% \\ 
        \midrule
    
        \Block{4-1}{KPRPE \\ \cite{kim2024keypoint}} & \Block{4-1}{IJB-S\textsuperscript{\textdagger} \\ \cite{kalka2018ijb}} & \Block{4-1}{Face} &
          N/A & \(<1\) & \(148\) & \(94.11\) & \(88.64\) & \(10.69\) & \(25.47\) \\
        & & & 5 & \(6.3\) & \(285\) & \(\mathbf{93.77}\) & \(88.81\) & \(21.88\) & \(45.27\) \\
        & & & 20 & \(23.4\) & \(595\) & \(94.50\) & \(\mathbf{88.57}\) & \(\mathbf{22.14}\) & \(\mathbf{45.62}\) \\
        & & & \(\infty\) & \(100\) & \(\) & \(94.49\) & \(89.12\) & \(21.76\) & \(45.46\) \\
        \midrule
    
        \Block{4-1}{SwinGait \\ \cite{fan2023exploring}} & 
          \Block{4-1}{Gait3D\textsuperscript{\textdagger} \\ \cite{zheng2022gait}} & \Block{4-1}{Gait} &
          N/A & \(<1\) & \(124\) & \(81.84\) & \(66.09\) & \(57.81\) & \(78.95\) \\
        & & & \(5\) & \(53.7\) & \(246\) & \(83.50\) & \(69.10\) & \(\mathbf{63.65}\) & \(78.95\) \\
        & & & \(20\) & \(92.3\) & \(319\) & \(81.01\) & \(65.43\) & \(\mathbf{63.65}\) & \(\mathbf{79.08}\) \\
        & & & \(\infty\) & \(100\) & \(338\) & \(\mathbf{78.80}\) & \(\mathbf{62.98}\) & \(\mathbf{63.65}\) & \(\mathbf{79.08}\) \\
        \midrule
    
        \Block{4-1}{BigGait \\ \cite{ye2024biggait}} & 
          \Block{4-1}{CCPG \\ \cite{li2023depth}} & \Block{4-1}{Gait} &
          N/A & \(<1\) & \(29\) & \(46.95\) & \(34.53\) & \(82.46\) & \(95.64\) \\
        & & & \(5\) & \(14.2\) & \(60\) & \(41.04\) & \(29.52\) & \(92.25\) & \(97.00\) \\
        & & & \(20\) & \(56.8\) & \(175\) & \(39.23\) & \(27.34\) & \(92.89\) & \(97.00\) \\
        & & & \(\infty\) & \(100\) & \(301\) & \(\mathbf{37.04}\) & \(\mathbf{25.34}\) & \(\mathbf{92.95}\) & \(\mathbf{97.01}\) \\
        \midrule
    
        \Block{4-1}{CAL} & \Block{4-1}{CCVID\textsuperscript{\textdagger} \\ \cite{gu2022clothes}} & \Block{4-1}{ReID} &
          N/A & \(<1\) & \(15\) & \(39.98\) & \(34.38\) & \(68.19\) & \(81.81\) \\
        & & & 5 & \(60.8\) & \(31\) & \(\mathbf{38.39}\) & \(\mathbf{33.72}\) & \(\mathbf{75.69}\) & \( \mathbf{85.34}\) \\
        & & & 20 & \(100\) & \(41\) & \(38.98\) & \(34.19\) & \(75.23\) & \(\mathbf{85.34}\) \\
        & & & \(\infty\) & \(100\) & \(41\) & \(38.98\) & \(34.19\) & \(75.23\) & \(\mathbf{85.34}\) \\
        \bottomrule
    \end{NiceTabular}
    \caption{Open-set FNIR@FPIR and verification TAR@FAR under different \(C_i\).
    \% Gal.: the ratio of the size of the clustered gallery to that of the full gallery.
    Runtime is reported in milliseconds (ms).
    \textsuperscript{\textdagger}FNIR@FPIR reported at 1\% and 5\% thresholds.}
    \label{tab:gallery_cluster_full_results}
\end{table*}

\paragraph{Relative Performance Improvement}
\cref{fig:relative_performance_diff} presents the results of \cref{tab:all_results} visually: the percentage of performance improvement compared to baselines on different modalities and datasets. \paper achieves significant performance improvements across all datasets, enhancing either verification TAR@0.1\%FAR or FNIR@1\%FPIR. Notably, it improves verification by approximately 100\% on the challenging IJB-S face recognition dataset and boosts gait recognition performance by about 90\% on BRIAR. Overall, \paper demonstrates greater performance gains in gait recognition and person reID, where open-set biometrics pose greater challenges, while also delivering improvements in face recognition. These results highlight the effectiveness and generalizability of our approach across diverse biometric tasks.

\subsection{1:1 Evaluation.}
The key distinction between the 1:1 and standard evaluation settings lies in the assumed availability of gallery data. Standard evaluation follows the conventional biometric protocol in which all enrolled subjects’ galleries are accessible for both verification and retrieval. This reflects typical operational scenarios, where a complete gallery is pre-enrolled prior to deployment; even in verification, galleries of non-claimed subjects are available, though they are not explicitly used by most systems. In extremely constrained environments, however, access to galleries of irrelevant subjects may be restricted. To assess \paper under such conditions, we additionally evaluate performance in the 1:1 setting, where only the gallery samples of the claimed identity are provided.

As illustrated in \cref{fig:single_subject}, this scenario yields a single column of per-subject scores that reflect the sole accessible gallery and fewer per-media scores compared to \cref{fig:toy_example}, since non-claimed identities’ media are excluded. Consequently, all indicators evaluate to 1, as each is computed via an argmax over a single value, and the \(k\)-NN scores are added to the per-subject scores. Despite these constraints, \cref{tab:all_results} shows that \paper continues to improve the performance of multiple biometric models across modalities in both open-set retrieval and verification.

\begin{table*}
    \centering\small
    \begin{NiceTabular}{lcccccccc}
        \toprule
        \Block{2-1}{Task} & \Block{2-1}{Model} & \Block{2-1}{Dataset} &
        \Block{2-1}{Mod.} &
        \Block{2-1}{Method} &
        \Block{1-2}{FNIR@FPIR $\downarrow$} & &
        \Block{1-2}{TAR@FAR $\uparrow$} & \\ 
        \cmidrule(lr){6-7} \cmidrule(lr){8-9}
        & & & & & 0.1\% & 1\% & 0.1\% & 1\% \\ 
        \midrule
    
        \Block{5-1}{Face} & 
            \Block{5-1}{KPRPE \\ \cite{kim2024keypoint}} & 
            \Block{5-1}{IJB-S\textsuperscript{\textdagger} \\ \cite{kalka2018ijb}} &
            \Block{5-1}{Face} &
            Max & \(95.39\) & \(89.80\) & \(8.73\) & \(22.40\) \\
        & & & & Min & \(99.83\) & \(99.11\) & \(1.40\) & \(7.45\) \\
        & & & & Mean & \(96.14\) & \(90.92\) & \(9.55\) & \(25.06\) \\
        & & & & R-R & \(95.32\) & \(92.45\) & \(8.69\) & \(23.98\) \\
        & & & & Ours & \(\mathbf{94.49}\) & \(\mathbf{89.12}\) & \(\mathbf{21.76}\) & \(\mathbf{45.47}\) \\
        \midrule
    
        \Block{5-1}{Gait} &
            \Block{5-1}{SwinGait \\ \cite{fan2023exploring}} &
            \Block{5-1}{Gait3D \\ \cite{zheng2022gait}} &
            \Block{5-1}{Gait} &
            Max & \(81.92\) & \(\mathbf{60.40}\) & \(70.66\) & \(91.11\) \\
        & & & & Min & \(88.02\) & \(73.24\) & \(50.99\) & \(78.34\) \\
        & & & & Mean & \(83.72\) & \(60.81\) & \(\mathbf{74.00}\) & \(\mathbf{93.73}\) \\
        & & & & R-R & \(82.84\) & \(67.61\) & \(56.96\) & \(77.98\) \\
        & & & & Ours & \(\mathbf{78.80}\) & \(62.98\) & \(63.65\) & \(79.08\) \\
        \midrule
    
        \Block{5-1}{Gait} & 
            \Block{5-1}{BigGait \\ \cite{ye2024biggait}}  & 
            \Block{5-1}{CCPG \\ \cite{li2023depth}} & 
            \Block{5-1}{Gait} &
            Max & \(42.87\) & \(28.92\) & \(86.94\) & \(\mathbf{97.11}\) \\
        & & & & Min & \(83.23\) & \(70.97\) & \(48.65\) & \(77.88\) \\
        & & & & Mean & \(51.08\) & \(38.12\) & \(79.22\) & \(95.20\) \\
        & & & & R-R & \(47.84\) & \(36.56\) & \(81.12\) & \(95.14\) \\
        & & & & Ours & \(\mathbf{37.04}\) & \(\mathbf{25.34}\) & \(\mathbf{92.95}\) & \(97.01\) \\
        \midrule
    
        \Block{5-1}{ReID} & 
            \Block{5-1}{CAL \\ \cite{gu2022clothes}} & 
            \Block{5-1}{CCVID\textsuperscript{\textdagger} \\ \cite{gu2022clothes}} &
            \Block{5-1}{ReID} &
            Max & \(40.55\) & \(\mathbf{32.89}\) & \(69.65\) & \(82.83\) \\
        & & & & Min & \(54.49\) & \(45.86\) & \(60.97\) & \(80.36\) \\
        & & & & Mean & \(\mathbf{38.74}\) & \(34.38\) & \(69.16\) & \(82.87\) \\
        & & & & R-R & \(40.52\) & \(34.12\) & \(68.88\) & \(83.60\) \\
        & & & & Ours & \(38.98\) & \(34.19\) & \(\mathbf{75.23}\) & \(\mathbf{85.34}\) \\
        \bottomrule
    \end{NiceTabular}
    \caption{Open-set FNIR@FPIR and verification TAR@FAR of multi-sample methods. \textsuperscript{\textdagger}FNIR@FPIR reported at 1\% and 5\% thresholds. [R-R: Re-ranking~\cite{zhong2017re}]}
    \label{tab:multi_sample_full_results}
\end{table*}

\subsection{Gallery Clustering}
\cref{tab:gallery_cluster_full_results} shows that increasing \(C_i\)—which corresponds to using a larger portion of the clustered gallery—generally improves both open-set and verification performance across all models and modalities, while introducing only moderate runtime overhead. In the baseline setting (\(C_i=\text{N/A}\)), where less than 1\% of the gallery is used, performance is consistently lower. As \(C_i\) increases, FNIR decreases and TAR increases, with most models achieving their best or near-best results when leveraging the full gallery (\(C_i=\infty\)). This trend is especially pronounced for gait and ReID models: BigGait on CCPG, for example, sees FNIR@0.1\% drop from \(46.95\%\) to \(37.04\%\) and TAR@0.1\% rise from \(82.46\%\) to \(92.95\%\). Similar improvements appear for SwinGait on Gait3D and CAL on CCVID. For face recognition on IJB-S, the gains plateau earlier, with \(C_i=20\) already matching or surpassing the full-gallery setting. Overall, the results indicate that incorporating a larger clustered gallery consistently enhances recognition performance across diverse biometric modalities, while the additional computation remains manageable for real-world deployment.

\subsection{Comparison to multi-sample methods}
\cref{tab:multi_sample_full_results} compares \paper with conventional multi-sample aggregation strategies across face, gait, and ReID tasks. Overall, the results show that pooling-based methods such as max, mean, and min exhibit inconsistent behavior: max often provides competitive FNIR but relatively weaker TAR, while mean generally improves verification performance but struggles in open-set conditions. Min performs worst across all settings, confirming its sensitivity to outliers. Re-ranking (R-R) offers moderate improvements but remains limited by its reliance on reciprocal neighbors and multiple probes. In contrast, \paper delivers strong and consistent gains, particularly in challenging open-set scenarios. For instance, on CCPG gait, \paper reduces FNIR@0.1\% FPIR to \(37.04\%\), markedly outperforming max (\(42.87\%\)) and mean (\(51.08\%\)), while also achieving the highest TAR@0.1\% FAR (\(92.95\%\)). Similar improvements appear in face and ReID settings, where \paper achieves the best or near-best performance across all metrics. These results highlight the advantage of incorporating local gallery structure into the scoring function, enabling more reliable recognition than traditional aggregation or re-ranking methods.

\subsection{Statistical Significant Tests}

\begin{table*}
    \centering\small
    \begin{NiceTabular}{lcccccc}
    \toprule
        \Block{2-1}{Model} & \Block{2-1}{Dataset} & \Block{2-1}{Gallery} & \Block{1-2}{FNIR@FPIR} & & \Block{1-2}{TAR@FAR} & \\ \cmidrule(lr){4-5} \cmidrule(lr){6-7}
        & & & 1\% & 5\% & 0.1\% & 1\% \\ \midrule
        \Block{2-1}{SwinGait} & \Block{2-1}{Gait3D} & Avg+Euc & \(0.0001\) & \(0.0001\) & \(0.0001\) & \(0.0001\) \\
         & & EVM & \(0.6441\) & \(0.3295\) & \(0.0001\) & \(0.3120\) \\ \midrule
        \Block{2-1}{BigGait} & \Block{2-1}{CCPG} & Avg+Euc & \(0.0001\) & \(0.0001\) & \(0.0001\) & \(0.0001\) \\
         & & EVM & \(0.0167\) & \(0.0014\) & \(0.0001\) & \(0.0001\) \\ \midrule
        \Block{2-1}{CAL} & \Block{2-1}{CCVID} & Avg+Cos & \(0.0001\) & \(0.4077\) & \(0.0001\) & \(0.0001\) \\
         & & EVM & \(1.0000\) & \(0.0012\) & \(0.0001\) & \(0.6741\) \\
        \bottomrule 
    \end{NiceTabular}
    \caption{FNIR@FPIR, TAR@FAR, and their statistical significance tests on real-world biometric datasets.}
    \label{tab:p_tests}
\end{table*}

We perform two-sample t-tests to evaluate whether the performance gains achieved with \paper are statistically significant across random mated-vs-non-mated splits, as shown in \cref{tab:p_tests}. The results indicate that all improvements in gait recognition are highly significant, with most \(p\)-values smaller than $0.0001$. Although the improvements on the CAL model for the CCVID dataset are less significant, \paper still achieves notable verification gains over CAL, which demonstrates its effectiveness. On other datasets, the large probe set sizes reduce random fluctuations. 

\begin{table*}
    \centering\small
    \begin{NiceTabular}{lccccccccc}
        \toprule
        \Block{2-1}{Modality} & \Block{2-1}{Model} & \Block{2-1}{Dataset} & \Block{2-1}{Gallery} & \Block{2-1}{Post-proc.} & \Block{1-2}{FNIR@FPIR $\downarrow$} & & \Block{1-2}{TAR@FAR $\uparrow$} & & \Block{2-1}{R@1 $\uparrow$} \\ \cmidrule(lr){6-7} \cmidrule(lr){8-9} 
        &  & & & & 0.1\% & 1\% & 0.1\% & 1\% \\ \midrule 
        \Block{6-1}{Face} & \Block{6-1}{KPRPE \\ \cite{kim2024keypoint}} & \Block{6-1}{BRIAR\textsuperscript{\textdaggerdbl} \\ \cite{cornett2023expanding}} & \Block{6-1}{CAFace \\ \cite{kim2022cluster}} & N/A & \(65.55\) & \(54.93\) & \(59.93\) & \(\mathbf{72.52}\) & \({84.34}\) \\
        & & & & +const. & \(65.55\) & \(55.00\) & \(61.64\) & \(72.39\) & \(84.34\) \\
        & & & & \(\times 2\) & \(65.55\) & \(55.00\) & \(61.64\) & \(\mathbf{72.52}\) & \(84.34\) \\
        & & & & 1,2-NN & \(64.99\) & \(\mathbf{54.85}\) & \(61.82\) & \(\mathbf{72.52}\) & \(84.34\) \\
        & & & & 1,2,3-NN & \(65.12\) & \(54.91\) & \(\mathbf{62.01}\) & \(\mathbf{72.52}\) & \(84.34\) \\
        & & & & Ours & \(\mathbf{64.81}\) & \({54.91}\) & \({61.89}\) & \(\mathbf{72.52}\) & \({84.34}\) \\ \midrule\midrule
        \Block{6-1}{Gait} & \Block{6-1}{SwinGait \\ \cite{fan2023exploring}} & \Block{6-1}{Gait3D\textsuperscript{\textdagger} \\ \cite{zheng2022gait}} & \Block{6-1}{Avg + Euc} & N/A & \(81.84 \smallpm 1.76\) & \(66.09 \smallpm 0.69\) & \(57.81\) & \(78.95\) & \({72.2}\) \\
        & & & & +const. & \(81.85 \smallpm 1.75\) & \(66.46 \smallpm 0.64\) & \(\mathbf{63.65}\) & \(\mathbf{79.08}\) & \(72.2\) \\
        & & & & \(\times 2\) & \(81.85 \smallpm 1.75\) & \(66.46 \smallpm 0.64\) & \(\mathbf{63.65}\) & \(\mathbf{79.08}\) & \(72.2\) \\
        & & & & 1,2-NN & \(80.19 \smallpm 1.85\) & \(64.69 \smallpm 0.61\) & \(\mathbf{63.65}\) & \(\mathbf{79.08}\) & \(72.2\) \\
        & & & & 1,2,3-NN & \(81.78 \smallpm 1.65\) & \(66.34 \smallpm 0.61\) & \(\mathbf{63.65}\) & \(\mathbf{79.08}\) & \(72.2\) \\
        & & & & Ours & \(\mathbf{78.80 \smallpm 1.99}\) & \(\mathbf{62.98 \smallpm 0.78}\) & \(\mathbf{63.65}\) & \(\mathbf{79.08}\) & \({72.2}\) \\
        \bottomrule
    \end{NiceTabular}
    \caption{Open-set FNIR@FPIR, verification TAR@FAR, and closed-set Rank-1 accuracy on real-world biometric datasets. 95\% confidence intervals are reported for FNIR@FPIR. \textsuperscript{\textdagger}FNIR@1\%FPIR and FNIR@5\%FPIR are reported. \textsuperscript{\textdaggerdbl}Rank-20 accuracy is reported according to the official evaluation protocol [Keys: R@1=rank-1 accuracy]}
    \label{tab:additional_results}
\end{table*}

\subsection{Effect of KNN scores}

We conduct additional experiments to analyze the effects of KNN scores on the original score matrices using several variations. First, we add a constant \(1\) to the highest score for each probe, which does not introduce additional discriminability between mated and non-mated probes (``+const.''). Second, we replace the highest score for each probe with twice its value, increasing its magnitude while maintaining limited discriminability between mated and non-mated probes but without leveraging KNN scores (``\(\times 2\)''). Third, we use the average of the 2 nearest neighbors (1,2-KNN) or the 3 nearest neighbors (1,2,3-KNN). 

As shown in \cref{tab:additional_results}, while all these approaches improve verification performance, none surpasses our proposed method. For open-set FNIR@FPIR, none of these variations outperforms \paper on Gait3D. Furthermore, on face recognition, the ``+const.'' and ``\(\times 2\)'' approaches fail to improve open-set performance, highlighting the necessity of incorporating KNN information for better open-set results. Including more nearest neighbors yields comparable results when all neighbors provide reasonable performance, as in face recognition. However, this strategy performs worse for gait recognition, where additional nearest neighbors contribute lower-quality information. In general, we conclude that KNN scores are both crucial and effective for improving verification and open-set performance.

\section{Additional Analyses}

\subsection{Additional Analyses on \cref{thm:main}}

Building on the analysis in \cref{fig:mc_sim_theorem_2_open_set}, we examine how \cref{eq:verification_theory} provides insights into real-world verification performance, focusing on the impact of \(\mu_3\). Using the same experimental setup as \cref{fig:mc_sim_theorem_2_score_matrix_open_set}, we analyze how variations in \(\mu_3\) influence performance, with \(\mu_3^{**}\) denoting the value where the LHS and RHS of \cref{eq:verification_theory} are equal. The results confirm that the theoretically predicted \(\mu_3^{**}\) closely aligns with the real-world value where applying \paper has no effect on performance. When \(\mu_3 < \mu_3^{**}\), \paper decreases verification performance, whereas for \(\mu_3 > \mu_3^{**}\), it improves performance.

\begin{figure}
    \centering
    
    \begin{subfigure}{0.45\linewidth}
        \centering
        \includegraphics[width=\linewidth]{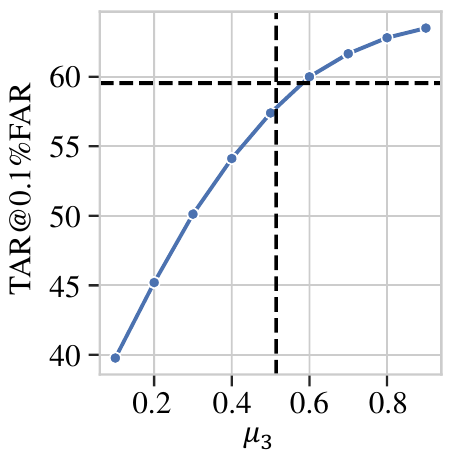}
        \caption{}
        \label{fig:mc_sim_theorem_2_score_matrix_verification}
    \end{subfigure}
    \hfill
    \begin{subfigure}{0.45\linewidth}
        \centering
        \includegraphics[width=\linewidth]{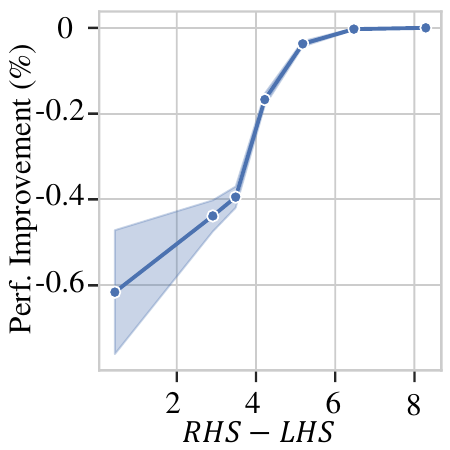}
        \caption{}
        \label{fig:mc_sim_theorem_2_features_verification}
    \end{subfigure}

    \caption{(a) The effect of \(\mu_3\) on TAR@0.1\%FAR. The vertical dashed line indicates the theoretically predicted \(\mu_3^{**}\), and the horizontal line indicates the performance without \paper. (b) The impact of the difference between the RHS and LHS of \cref{eq:verification_theory} on the improvement of TAR@0.1\%FAR by using \paper. The shaded area is the \(99\%\) confidence interval.}
    \label{fig:mc_sim_theorem_2_verification}
\end{figure}

\cref{fig:mc_sim_theorem_2_features_verification} shows how the difference between the RHS and LHS of \cref{eq:verification_theory} influences TAR@0.1\%FAR, with larger differences generally leading to greater improvements. However, verification performance slightly declines with KNN, likely due to a larger offset between actual and theoretical \(\mu_3\) values in \cref{fig:mc_sim_theorem_2_score_matrix_verification}, where no improvement is predicted. This may result from the stronger impact of the normality assumptions in \cref{eq:verification_theory} on verification compared to open-set biometrics, as also observed in \cref{fig:synthetic_exp_verification}, where KNN often worsens performance. Despite this, \paper consistently enhances verification on real-world biometric datasets, as demonstrated in \cref{fig:relative_performance_diff}, with significant improvements except for gait recognition on BRIAR. This indicates that while \cref{eq:open_set_theory} offers practical guidance for real-world scenarios, \cref{eq:verification_theory} is less effective in predicting performance improvements with \paper.

\subsection{Additional Ablation Study}

\begin{figure}
    \centering
    
    \begin{subfigure}{0.45\linewidth}
        \centering
        \includegraphics[width=\linewidth]{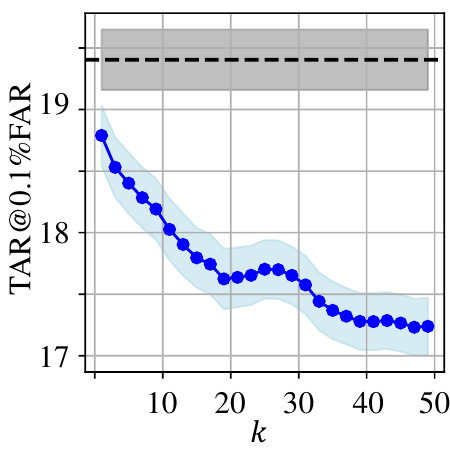}
        \caption{The effect of \(k\).}
        \label{fig:verification_vs_k}
    \end{subfigure}
    \hfill
    \begin{subfigure}{0.45\linewidth}
        \centering
        \includegraphics[width=\linewidth]{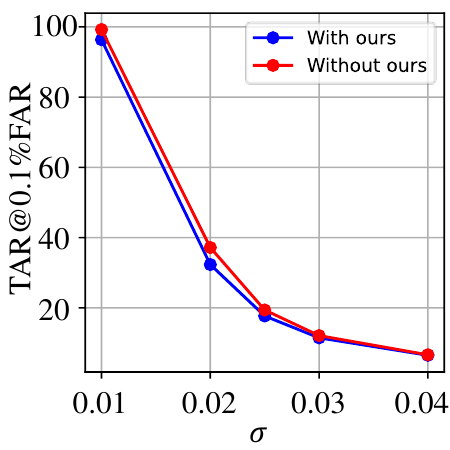}
        \caption{The effect of \(\sigma\).}
        \label{fig:verification_vs_sigma_i}
    \end{subfigure}

    \caption{The effect of \(k\) and \(\sigma\) on TAR@0.1\%FAR. In (a), the horizontal line denotes performance without \paper. The shaded area is the \(99\%\) confidence interval.}
    \label{fig:synthetic_exp_verification}
\end{figure}

\cref{fig:synthetic_exp_verification} illustrates the effect of \paper on verification performance using synthetic datasets. In \cref{fig:verification_vs_k}, TAR@0.1\%FAR shows a gradual decrease as \(k\) increases, but the decrease is minimal—at worst about 2\%—compared to the significant improvement of FNIR@1\%FPIR in Fig. 4(a), which reaches approximately 15\%. This demonstrates that \paper substantially enhances open-set performance with only slightly reducing verification performance on randomly generated features.

\cref{fig:verification_vs_sigma_i} examines the impact of \(\sigma\), which governs intra-class variation, on verification performance. As \(\sigma\) increases, intra-class variation grows, and the performance gap between using and not using \paper narrows. Importantly, our approach achieves comparable verification performance to the baseline, even with higher \(\sigma\). This demonstrates that \paper enhances open-set performance without compromising verification, providing a balanced and comprehensive evaluation of biometric systems.

\section{Limitations}
\paper assumes that gallery media are well-distributed and provide a reasonable approximation of the actual gallery feature distribution. While this assumption may not always hold in real-world scenarios, we introduce a theoretical framework (\cref{thm:main}) that predicts whether \paper will improve verification and open-set biometric performance. We validate this framework on real-world datasets as well as on randomly generated features and scores. In practice, a small test set can be used to assess the potential performance gains of \paper before deployment.

Our scoring algorithm requires additional storage for per-media features, but this demand is manageable due to several factors. Biometric features are typically compact, and high-capacity storage is increasingly affordable. Additionally, clustering the gallery samples can reduce storage load without significantly sacrificing performance, and cloud-based storage is scalable and cost-effective in handling data without local limitations. 

\section{Potential Societal Impacts}

While our work advances biometric technologies with potential benefits in security and forensics, it also raises concerns about privacy and individual freedoms. Sensitive biometric data, such as facial features, could be misused for unauthorized surveillance or profiling, which severely erodes public trust and infringes on rights. We prioritize using data collected with informed consent and avoid collecting any biometric data directly. We urge the community to adopt transparent and responsible practices to ensure that biometrics serve society positively while minimizing potential harm.

%% file: sec/refs.bib
@String(CVPR= {IEEE Conf. Comput. Vis. Pattern Recog.})

@String(ICCV= {Int. Conf. Comput. Vis.})

@String(ECCV= {Eur. Conf. Comput. Vis.})

@String(ICLR = {Int. Conf. Learn. Represent.})

@String(CVPR  = {CVPR})

@String(ICCV  = {ICCV})

@String(ECCV  = {ECCV})

@String(ICLR  = {ICLR})

@inproceedings{ye2024biggait,
  title={BigGait: Learning Gait Representation You Want by Large Vision Models},
  author={Ye, Dingqiang and Fan, Chao and Ma, Jingzhe and Liu, Xiaoming and Yu, Shiqi},
  booktitle={{CVPR}},
  year={2024}
}

@inproceedings{kim2022cluster,
  title={Cluster and aggregate: Face recognition with large probe set},
  author={Kim, Minchul and Liu, Feng and Jain, Anil K and Liu, Xiaoming},
  booktitle={{NeurIPS}},
  year={2022}
}

@inproceedings{kim2024keypoint,
  title={KeyPoint Relative Position Encoding for Face Recognition},
  author={Kim, Minchul and Su, Yiyang and Liu, Feng and Jain, Anil and Liu, Xiaoming},
  booktitle={{CVPR}},
  year={2024}
}

@inproceedings{kim2022adaface,
  title={Adaface: Quality adaptive margin for face recognition},
  author={Kim, Minchul and Jain, Anil K and Liu, Xiaoming},
  booktitle={{CVPR}},
  year={2022}
}

@inproceedings{gunther2017toward,
  title={Toward open-set face recognition},
  author={G\"{u}nther, Manuel and Cruz, Steve and Rudd, Ethan M and Boult, Terrance E},
  booktitle={CVPR Workshops},
  year={2017}
}

@misc{grother2019face,
  title={Face recognition vendor test ({FRVT}) part 2: Identification},
  author={Grother, Patrick and Ngan, Mei and Hanaoka, Kayee},
  year={2019},
  publisher={US Department of Commerce, National Institute of Standards and Technology~…},
  url={https://www.nist.gov/document/nistir827120190911pdf},
  urldate={2024-11-14}
}

@inproceedings{kalka2018ijb,
  title={{IJB-S:} {IARPA} Janus Surveillance Video Benchmark},
  author={Kalka, Nathan D and Maze, Brianna and Duncan, James A and O’Connor, Kevin and Elliott, Stephen and Hebert, Kaleb and Bryan, Julia and Jain, Anil K},
  booktitle={BTAS},
  year={2018},
}

@Inbook{Jain2025,
author="Jain, Anil K.
and Ross, Arun A.
and Nandakumar, Karthik
and Swearingen, Thomas",
title="Face Recognition",
bookTitle="Introduction to Biometrics",
year="2025",
publisher="Springer International Publishing",
pages="119--173",
}

@inproceedings{zheng2022gait,
  title={Gait recognition in the wild with dense 3d representations and a benchmark},
  author={Zheng, Jinkai and Liu, Xinchen and Liu, Wu and He, Lingxiao and Yan, Chenggang and Mei, Tao},
  booktitle={{CVPR}},
  year={2022}
}

@article{fan2023exploring,
  title={Exploring Deep Models for Practical Gait Recognition},
  author={Fan, Chao and Hou, Saihui and Huang, Yongzhen and Yu, Shiqi},
  journal={arXiv preprint arXiv:2303.03301},
  year={2023}
}

@inproceedings{li2023depth,
  title={An In-Depth Exploration of Person Re-Identification and Gait Recognition in Cloth-Changing Conditions},
  author={Li, Weijia and Hou, Saihui and Zhang, Chunjie and Cao, Chunshui and Liu, Xu and Huang, Yongzhen and Zhao, Yao},
  booktitle={CVPR},
  year={2023}
}

@article{shen2022comprehensive,
  title={A comprehensive survey on deep gait recognition: algorithms, datasets and challenges},
  author={Shen, Chuanfu and Yu, Shiqi and Wang, Jilong and Huang, George Q and Wang, Liang},
  journal={arXiv preprint arXiv:2206.13732},
  year={2022}
}

@inproceedings{gu2022clothes,
  title={Clothes-changing person re-identification with rgb modality only},
  author={Gu, Xinqian and Chang, Hong and Ma, Bingpeng and Bai, Shutao and Shan, Shiguang and Chen, Xilin},
  booktitle={CVPR},
  year={2022}
}

@InProceedings{liu2024distilling,
    author    = {Liu, Feng and Kim, Minchul and Ren, Zhiyuan and Liu, Xiaoming},
    title     = {Distilling {CLIP} with Dual Guidance for Learning Discriminative Human Body Shape Representation},
    booktitle = {{CVPR}},
    year      = {2024}
}

@article{hermans2017defense,
  title={In defense of the triplet loss for person re-identification},
  author={Hermans, Alexander and Beyer, Lucas and Leibe, Bastian},
  journal={arXiv preprint arXiv:1703.07737},
  year={2017}
}

@inproceedings{vaze2022openset,
title={Open-Set Recognition: A Good Closed-Set Classifier is All You Need},
author={Sagar Vaze and Kai Han and Andrea Vedaldi and Andrew Zisserman},
booktitle={{ICLR}},
year={2022},
}

@inproceedings{kong2021opengan,
  title={Opengan: Open-set recognition via open data generation},
  author={Kong, Shu and Ramanan, Deva},
  booktitle={{ICCV}},
  year={2021}
}

@inproceedings{yang2022openood,
  title={Openood: Benchmarking generalized out-of-distribution detection},
  author={Yang, Jingkang and Wang, Pengyun and Zou, Dejian and Zhou, Zitang and Ding, Kunyuan and Peng, Wenxuan and Wang, Haoqi and Chen, Guangyao and Li, Bo and Sun, Yiyou and others},
  booktitle={{NeurIPS}},
  year={2022}
}

@inproceedings{wang2022openauc,
title={Open{AUC}: Towards {AUC}-Oriented Open-Set Recognition},
author={Zitai Wang and Qianqian Xu and Zhiyong Yang and Yuan He and Xiaochun Cao and Qingming Huang},
booktitle={{NeurIPS}},
year={2022},
}

@inproceedings{bendale2016towards,
  title={Towards open set deep networks},
  author={Bendale, Abhijit and Boult, Terrance E},
  booktitle={{CVPR}},
  year={2016}
}

@inproceedings{kuchibhotla2022unseen,
  title={Unseen Classes at a Later Time? No Problem},
  author={Kuchibhotla, Hari Chandana and Malagi, Sumitra S and Chandhok, Shivam and Balasubramanian, Vineeth N},
  booktitle={{CVPR}},
  year={2022}
}

@inproceedings{sun2022out,
  title={Out-of-distribution detection with deep nearest neighbors},
  author={Sun, Yiyou and Ming, Yifei and Zhu, Xiaojin and Li, Yixuan},
  booktitle={{ICML}},
  year={2022},
}

@inproceedings{cruz2024operational,
  title={Operational Open-Set Recognition and PostMax Refinement},
  author={Cruz, Steve and Rabinowitz, Ryan and G{\"u}nther, Manuel and Boult, Terrance E.},
  booktitle={{ECCV}},
  year={2024}
}

@article{zahra2023person,
  title={Person re-identification: A retrospective on domain specific open challenges and future trends},
  author={Zahra, Asmat and Perwaiz, Nazia and Shahzad, Muhammad and Fraz, Muhammad Moazam},
  journal={Pattern Recognition},
  year={2023},
}

@article{geng2020recent,
  title={Recent advances in open set recognition: A survey},
  author={Geng, Chuanxing and Huang, Sheng-jun and Chen, Songcan},
  journal={IEEE TPAMI},
  year={2020},
}

@InProceedings{su2024open,
    title={Open-Set Biometrics: Beyond Good Closed-Set Models},
    author={Su, Yiyang and Kim, Minchul and Liu, Feng and Jain, Anil and Liu, Xiaoming},
    booktitle = {{{ECCV}}},
    year      = {2024},
}

@inproceedings{cornett2023expanding,
  title={Expanding accurate person recognition to new altitudes and ranges: The briar dataset},
  author={Cornett, David and Brogan, Joel and Barber, Nell and Aykac, Deniz and Baird, Seth and Burchfield, Nicholas and Dukes, Carl and Duncan, Andrew and Ferrell, Regina and Goddard, Jim and others},
  booktitle={WACV},
  year={2023}
}

@inproceedings{zhong2017re,
  title={Re-ranking person re-identification with k-reciprocal encoding},
  author={Zhong, Zhun and Zheng, Liang and Cao, Donglin and Li, Shaozi},
  booktitle={{CVPR}},
  year={2017}
}

@inproceedings{liu2024farsight,
  title={Farsight: A physics-driven whole-body biometric system at large distance and altitude},
  author={Liu, Feng and Ashbaugh, Ryan and Chimitt, Nicholas and Hassan, Najmul and Hassani, Ali and Jaiswal, Ajay and Kim, Minchul and Mao, Zhiyuan and Perry, Christopher and Ren, Zhiyuan and others},
  booktitle={{WACV}},
  year={2024}
}

@inproceedings{bai2019re,
    author    = {Bai, Song and Tang, Peng and Torr, Philip H.S. and Latecki, Longin Jan},
    title     = {Re-Ranking via Metric Fusion for Object Retrieval and Person Re-Identification},
    booktitle = {{CVPR}},
    year      = {2019}
}

@inproceedings{tan2021instance,
  title={Instance-level image retrieval using reranking transformers},
  author={Tan, Fuwen and Yuan, Jiangbo and Ordonez, Vicente},
  booktitle={{ICCV}},
  year={2021}
}

@inproceedings{sarfraz2018pose,
  title={A pose-sensitive embedding for person re-identification with expanded cross neighborhood re-ranking},
  author={Sarfraz, M Saquib and Schumann, Arne and Eberle, Andreas and Stiefelhagen, Rainer},
  booktitle={{CVPR}},
  year={2018}
}

@inproceedings{shen2021re,
  title={Re-ranking for image retrieval and transductive few-shot classification},
  author={Shen, Xi and Xiao, Yang and Hu, Shell Xu and Sbai, Othman and Aubry, Mathieu},
  booktitle={NeurIPS},
  year={2021}
}

@inproceedings{ouyang2021contextual,
  title={Contextual similarity aggregation with self-attention for visual re-ranking},
  author={Ouyang, Jianbo and Wu, Hui and Wang, Min and Zhou, Wengang and Li, Houqiang},
  booktitle={NeurIPS},
  year={2021}
}

@inproceedings{gunther2017unconstrained,
  title={Unconstrained face detection and open-set face recognition challenge},
  author={G{\"u}nther, Manuel and Hu, Peiyun and Herrmann, Christian and Chan, Chi-Ho and Jiang, Min and Yang, Shufan and Dhamija, Akshay Raj and Ramanan, Deva and Beyerer, J{\"u}rgen and Kittler, Josef and others},
  booktitle={{IJCB}},
  year={2017},
  organization={IEEE}
}

@inproceedings{ross2019some,
  title={Some research problems in biometrics: The future beckons},
  author={Ross, Arun and Banerjee, Sudipta and Chen, Cunjian and Chowdhury, Anurag and Mirjalili, Vahid and Sharma, Renu and Swearingen, Thomas and Yadav, Shivangi},
  booktitle={{ICB}},
  year={2019},
}

@article{gnedenko1943distribution,
  title={Sur la distribution limite du terme maximum d'une serie aleatoire},
  author={Gnedenko, Boris},
  journal={Annals of mathematics},
  year={1943},
}

@inproceedings{fisher1928limiting,
  title={Limiting forms of the frequency distribution of the largest or smallest member of a sample},
  author={Fisher, Ronald Aylmer and Tippett, Leonard Henry Caleb},
  booktitle={Mathematical proceedings of the Cambridge philosophical society},
  year={1928},
}

@article{elfving1947asymptotical,
  title={The asymptotical distribution of range in samples from a normal population},
  author={Elfving, Gustav},
  journal={Biometrika},
  year={1947},
}

@article{jain2021biometrics,
  title={Biometrics: Trust, but verify},
  author={Jain, Anil K and Deb, Debayan and Engelsma, Joshua J},
  journal={IEEE Transactions on Biometrics, Behavior, and Identity Science},
  year={2021},
}

@book{ross2006handbook,
  title={Handbook of multibiometrics},
  author={Ross, Arun A and Jain, Anil K and Nandakumar, Karthik},
  year={2006},
  publisher={Springer}
}

@inproceedings{poh2002multi,
  title={A multi-sample multi-source model for biometric authentication},
  author={Poh, Norman and Bengio, Samy and Korczak, Jerzy},
  booktitle={Proceedings of the 12th ieee workshop on neural networks for signal processing},
  year={2002},
  organization={IEEE}
}

@inproceedings{ebrahimi2025gif,
  title={GIF: Generative Inspiration for Face Recognition at Scale},
  author={Ebrahimi, Saeed and Rahimi, Sahar and Dabouei, Ali and Das, Srinjoy and Dawson, Jeremy M and Nasrabadi, Nasser M},
  booktitle={Proceedings of the Computer Vision and Pattern Recognition Conference},
  pages={3528--3539},
  year={2025}
}

@inproceedings{jin2025denoising,
  title={On Denoising Walking Videos for Gait Recognition},
  author={Jin, Dongyang and Fan, Chao and Ma, Jingzhe and Zhou, Jingkai and Chen, Weihua and Yu, Shiqi},
  booktitle={{CVPR}},
  year={2025}
}

@inproceedings{shen2025lidargait++,
  title={LidarGait++: Learning Local Features and Size Awareness from LiDAR Point Clouds for 3D Gait Recognition},
  author={Shen, Chuanfu and Wang, Rui and Duan, Lixin and Yu, Shiqi},
  booktitle={{CVPR}},
  year={2025}
}

@inproceedings{yang2025bridging,
  title={Bridging gait recognition and large language models sequence modeling},
  author={Yang, Shaopeng and Wang, Jilong and Hou, Saihui and Liu, Xu and Cao, Chunshui and Wang, Liang and Huang, Yongzhen},
  booktitle={{CVPR}},
  year={2025}
}

@inproceedings{wang2025secap,
  title={SeCap: Self-Calibrating and Adaptive Prompts for Cross-view Person Re-Identification in Aerial-Ground Networks},
  author={Wang, Shining and Wang, Yunlong and Wu, Ruiqi and Jiao, Bingliang and Wang, Wenxuan and Wang, Peng},
  booktitle={{CVPR}},
  year={2025}
}

@inproceedings{cui2025dkc,
  title={DKC: Differentiated Knowledge Consolidation for Cloth-Hybrid Lifelong Person Re-identification},
  author={Cui, Zhenyu and Zhou, Jiahuan and Peng, Yuxin},
  booktitle={{CVPR}},
  year={2025}
}

@inproceedings{liang2025differ,
  title={DIFFER: Disentangling Identity Features via Semantic Cues for Clothes-Changing Person Re-ID},
  author={Liang, Xin and Rawat, Yogesh S},
  booktitle={{CVPR}},
  year={2025}
}

@inproceedings{yuan2025poses,
  title={From poses to identity: Training-free person re-identification via feature centralization},
  author={Yuan, Chao and Zhang, Guiwei and Ma, Changxiao and Zhang, Tianyi and Niu, Guanglin},
  booktitle={{CVPR}},
  year={2025}
}

@inproceedings{yang2025cheb,
  title={Cheb-GR: Rethinking K-nearest Neighbor Search in Re-ranking for Person Re-identification},
  author={Yang, Jinxi and Li, He and Du, Bo and Ye, Mang},
  booktitle={CVPR},
  year={2025}
}

@inproceedings{pang2025identity,
  title={Identity-Clothing Similarity Modeling for Unsupervised Clothing Change Person Re-Identification},
  author={Pang, Zhiqi and Wang, Junjie and Zhao, Lingling and Wang, Chunyu},
  booktitle={{CVPR}},
  year={2025}
}

@inproceedings{chen2025learning,
  title={Learning from Synchronization: Self-Supervised Uncalibrated Multi-View Person Association in Challenging Scenes},
  author={Chen, Keqi and Srivastav, Vinkle and Mutter, Didier and Padoy, Nicolas},
  booktitle={{CVPR}},
  year={2025}
}

@inproceedings{you2025lvface,
  title={LVFace: Progressive Cluster Optimization for Large Vision Models in Face Recognition},
  author={You, Jinghan and Li, Shanglin and Sun, Yuanrui and Wei, Jiangchuan and Guo, Mingyu and Feng, Chao and Ran, Jiao},
  booktitle={{ICCV}},
  year={2025}
}

@inproceedings{peng2025stylized,
  title={Stylized-Face: A Million-level Stylized Face Dataset for Face Recognition},
  author={Peng, Zhengyuan and Xu, Jianqing and Huang, Yuge and Hao, Jinkun and Ding, Shouhong and Zhang, Zhizhong and Tan, Xin and Ma, Lizhuang},
  booktitle={ICCV},
  year={2025}
}

@inproceedings{kim2025vigface,
  title={VIGFace: Virtual Identity Generation for Privacy-Free Face Recognition Dataset},
  author={Kim, Minsoo and Sagong, Min-Cheol and Nam, Gi Pyo and Cho, Junghyun and Kim, Ig-Jae},
  booktitle={Proceedings of the IEEE/CVF International Conference on Computer Vision},
  pages={10043--10053},
  year={2025}
}

@inproceedings{kim2025idface,
  title={IDFace: Face Template Protection for Efficient and Secure Identification},
  author={Kim, Sunpill and Paik, Seunghun and Hwang, Chanwoo and Kim, Dongsoo and Shin, Junbum and Seo, Jae Hong},
  booktitle={{CVPR}},
  year={2025}
}

@inproceedings{ye2025biggergait,
  title={Biggergait: Unlocking gait recognition with layer-wise representations from large vision models},
  author={Ye, Dingqiang and Fan, Chao and Huang, Zhanbo and Luo, Chengwen and Li, Jianqiang and Yu, Shiqi and Liu, Xiaoming},
  booktitle={{NeurIPS}},
  year={2025}
}

@inproceedings{habib2025cargait,
  title={Cargait: Cross-attention based re-ranking for gait recognition},
  author={Habib, Gavriel and Barzilay, Noa and Shimshi, Or and Ben-Ari, Rami and Darshan, Nir},
  booktitle={{ICCV}},
  year={2025}
}

@inproceedings{wang2025gait,
  title={Gait-X: Exploring X modality for Generalized Gait Recognition},
  author={Wang, Zengbin and Hou, Saihui and Li, Junjie and Liu, Xu and Cao, Chunshui and Huang, Yongzhen and Wang, Siye and Zhang, Man},
  booktitle={{ICCV}},
  year={2025}
}

@inproceedings{huang2025learning,
  title={Learning A Unified Template for Gait Recognition},
  author={Huang, Panjian and Hou, Saihui and Huang, Junzhou and Huang, Yongzhen},
  booktitle={{ICCV}},
  year={2025}
}

@inproceedings{khalid2025bridging,
  title={Bridging the Sky and Ground: Towards View-Invariant Feature Learning for Aerial-Ground Person Re-Identification},
  author={Khalid, Wajahat and Liu, Bin and Li, Xulin and Waqas, Muhammad and Afgan, Muhammad Sher},
  booktitle={{ICCV}},
  year={2025}
}

@inproceedings{li2025one,
  title={One-Shot Knowledge Transfer for Scalable Person Re-Identification},
  author={Li, Longhua and Qi, Lei and Geng, Xin},
  booktitle={{ICCV}},
  year={2025}
}

@inproceedings{zhang2025viperson,
  title={VIPerson: Flexibly Generating Virtual Identity for Person Re-Identification},
  author={Zhang, Xiao-Wen and Zhang, Delong and Peng, Yi-Xing and Ouyang, Zhi and Meng, Jingke and Zheng, Wei-Shi},
  booktitle={{ICCV}},
  year={2025}
}

@inproceedings{su2025hamobe,
  title={Hamobe: Hierarchical and adaptive mixture of biometric experts for video-based person reid},
  author={Su, Yiyang and Shi, Yunping and Liu, Feng and Liu, Xiaoming},
  booktitle={{ICCV}},
  year={2025}
}

@inproceedings{xu2025self,
  title={Self-Reinforcing Prototype Evolution with Dual-Knowledge Cooperation for Semi-Supervised Lifelong Person Re-Identification},
  author={Xu, Kunlun and Zhuo, Fan and Li, Jiangmeng and Zou, Xu and Zhou, Jiahuan},
  booktitle={{ICCV}},
  year={2025}
}

@article{zhang2023graph,
  title={Graph convolution based efficient re-ranking for visual retrieval},
  author={Zhang, Yuqi and Qian, Qi and Wang, Hongsong and Liu, Chong and Chen, Weihua and Wang, Fan},
  journal={IEEE Transactions on Multimedia},
  year={2023},
}

@inproceedings{barman2017shape,
  title={Shape: A novel graph theoretic algorithm for making consensus-based decisions in person re-identification systems},
  author={Barman, Arko and Shah, Shishir K},
  booktitle={{ICCV}},
  year={2017}
}

@article{best2014unconstrained,
  title={Unconstrained face recognition: Identifying a person of interest from a media collection},
  author={Best-Rowden, Lacey and Han, Hu and Otto, Charles and Klare, Brendan F and Jain, Anil K},
  journal={IEEE Transactions on Information Forensics and Security},
  year={2014},
}

@article{liu2025person,
  title={Person Recognition at Altitude and Range: Fusion of Face, Body Shape and Gait},
  author={Liu, Feng and Chimitt, Nicholas and Guo, Lanqing and Jain, Jitesh and Kane, Aditya and Kim, Minchul and Robbins, Wes and Su, Yiyang and Ye, Dingqiang and Zhang, Xingguang and others},
  journal={arXiv preprint arXiv:2505.04616},
  year={2025}
}

@inproceedings{zhu2025quality,
  title={A quality-guided mixture of score-fusion experts framework for human recognition},
  author={Zhu, Jie and Su, Yiyang and Kim, Minchul and Jain, Anil and Liu, Xiaoming},
  booktitle={{ICCV}},
  year={2025}
}

@inproceedings{kim2025sapiensid,
  title={Sapiensid: Foundation for human recognition},
  author={Kim, Minchul and Ye, Dingqiang and Su, Yiyang and Liu, Feng and Liu, Xiaoming},
  booktitle={{CVPR}},
  year={2025}
}

@article{ren2023chatgpt,
  title={ChatGPT-powered hierarchical comparisons for image classification},
  author={Ren, Zhiyuan and Su, Yiyang and Liu, Xiaoming},
  journal={Advances in neural information processing systems},
  volume={36},
  pages={69706--69718},
  year={2023}
}
